%% file: ms.tex
\RequirePackage{fix-cm}
\pdfoutput=1
\documentclass[smallcondensed,natbib]{svjour3}     
\smartqed  
\include{header}
\usepackage[many]{tcolorbox}
\tcbset{
  myhlight/.style={
    colback=cyan!10,
    arc=0pt,
    outer arc=0pt,
    boxrule=0pt,
    top=2pt,
    bottom=2pt,
    left=2pt,
    right=2pt,
  },
  highlight math style={myhlight},
  mybx/.style={
    colback=white,
    arc=0pt,
    outer arc=0pt,
    boxrule=1pt,
    top=2pt,
    bottom=2pt,
    left=2pt,
    right=2pt,
  }
}

\newtcolorbox{mybox}{
	boxsep=1pt,
  	breakable,
  	mybx
}

\newtcolorbox{highlight}{
	boxsep=0pt,
  	breakable,
  	myhlight
}

\newtcbox{inlight}{on line, boxsep=0pt,
  breakable,
  myhlight
}

\begin{document}

\title{The PRIMPing Routine --- Tiling through Proximal Alternating Linearized Minimization}

\titlerunning{The PRIMPing Routine}        

\author{Sibylle Hess         \and
        Katharina Morik 	\and\\
        Nico Piatkowski
}


\institute{TU Dortmund, Computer Science, LS 8, 44221 Dortmund, Germany \\
}

\date{Received: date / Accepted: date}

\maketitle

\begin{abstract}
Mining and exploring databases should provide users with knowledge and new insights. 
Tiles of data strive to unveil true underlying structure and distinguish valuable information from various kinds of noise.
We propose a novel Boolean matrix factorization algorithm to solve the tiling problem, based on recent results from optimization theory.
In contrast to existing work, the new algorithm minimizes the description length of the resulting factorization. This approach is well known for model selection and data compression, but not for finding suitable factorizations via numerical optimization. 
We demonstrate the superior robustness of the new approach in the presence of several kinds of noise and types of underlying structure. 
Moreover, our general framework can work with any cost measure having a suitable real-valued relaxation. Thereby, no convexity assumptions have to be met. 
The experimental results on synthetic data and image data show that the new method identifies interpretable patterns 
which explain the data almost always better than the competing algorithms. 

\keywords{Tiling \and Boolean Matrix Factorization \and Minimum Description Length principle \and Proximal Alternating Linearized Minimization \and Nonconvex-Nonsmooth Minimization \and Alternating Minimization}
\end{abstract}
\section{Introduction}
\label{intro}
In a large range of data mining tasks such as Market Basket Analysis, Text Mining, Collaborative Filtering or DNA Expression Analysis, we are interested in the exploration of data which is represented by a binary matrix.  Data exploration is unsupervised by nature; the objective is to gain insight by a summarization of its relevant parts. Here, we seek for sets of columns and rows whose intersecting positions frequently feature a one. This identifies, e.g., groups of users together with their shared preferences, genes that are often co-expressed among several tissue samples, or words that occur together in documents describing the same topic. The identification of such sets of columns and rows is studied from the perspective of various data mining subfields as \emph{biclustering}, \emph{tiling} or \emph{matrix factorization} \citep{tatti2012comparing,zimek2013blind}.

Consider the example binary database presented on the left in  Fig.~\ref{fig:patternMining}. The distribution of ones appears disarrayed, but a suitable permutation of columns and rows reveals a formation of blocks, depicted by the matrix on the right. Interpreting all binary entries which do not fit to this formation as noise, we aim to separate the haphazard component from  the formative one. 
 
\begin{figure}
\centering
\includegraphics[angle=90,width = 0.475\columnwidth,trim=0.5cm 0.5cm 0.5cm 0.5cm, clip=true]{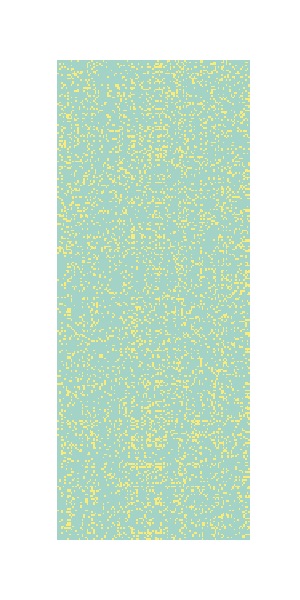}
\includegraphics[angle=90,width = 0.475\columnwidth,trim=0.5cm 0.5cm 0.5cm 0.5cm, clip=true]{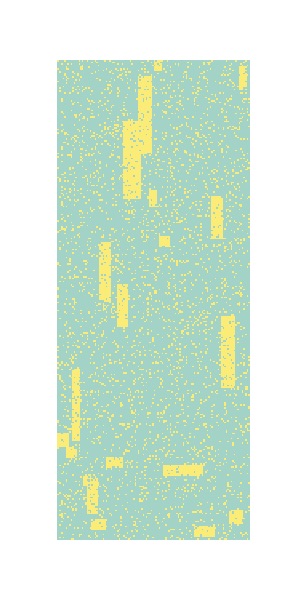}
\caption{Example binary dataset (left) with ones in yellow and zeros in blue. A rearrangement of columns and rows unveils structure (right).\label{fig:patternMining}}
\end{figure}

This objective is difficult to delineate: where to draw the line between structure and noise? Are there natural limitations on the amount of blocks to derive? To what extend may they overlap? \cite{Miettinen2014mdl4bmf} successfully apply the \emph{Minimum Description Length} (MDL) principle to reduce these considerations into one objective: exploit just as many regularities as serves the compression of the data. Identifying regularities with column-row interrelations, the description length counterbalances the complexity of the model (derived interrelations) and the fit to the data, measured by the size of the encoded data using the model. Decisive for the feasibility of extracted components is the definition of the encoding. 

\cite{Miettinen2014mdl4bmf} evaluate several encodings with respect to their ability to filter a planted structure from noise. The method they use applies a greedy Boolean matrix factorization to extract candidate interrelations which are selected according to a specified description length. Another framework proposed by \cite{lucchese2014unifying} greedily selects the interrelations directly in accordance with the description length. Most recently, \cite{nassau15} propose another greedy algorithm with focus on a setting where ones more probably indicate interrelations than noise. All these methods are capable to identify the underlying structure in respectively examined settings. All in all, the experiments indicate however that the quality considerably varies depending on the distribution of noise and characteristics of the dataset \citep{Miettinen2014mdl4bmf,nassau15}.  

For real-world datasets, it is difficult (if not impossible) to estimate these aspects, in order to choose the appropriate algorithm or to assess its quality on the given dataset. Believing that the unsteady performance is due to a lack of theoretical foundation, we introduce a framework called \emph{PAL-Tiling} to numerically optimize a cost measure which has a suitable real-valued approximation. In this respect, we derive approximations of two MDL cost measures, consequently proposing two algorithms: one applying $L1$-regularization on the matrix factorization (\textsc{Panpal}) and one employing an encoding by code tables as proposed by \cite{siebes2006item} (\textsc{Primp}). 
We assess the algorithms' ability to filter the \textit{true} underlying structure from the noise. Therefore, we compare various performance measures in a controlled setting of synthetically generated data as well as for real-world data. We show that \textsc{Primp} is capable of recovering the latent structure in spite of varying database characteristics and noise distributions. In addition, we visualize the derived categorization into tiles by means of images, showing that our conducted minimization procedure of \textit{PAL-Tiling} yields interpretable groupings. 
\subsection{Roadmap}
In Section \ref{sec:background} we introduce our notation and review the work related to the three research branches of Tiling, MDL, and Nonnegative Matrix Factorization, which compound our method. After surveying these building blocks, we introduce our optimization framework \textsc{PAL-Tiling} in Sec. \ref{sec:PALTiling}. We derive the proximal mapping with respect to the proposed penalization of non-binary values, enabling the minimization of the approximate Boolean matrix factorization error under convergence guarantees. Therewith we derive the $L1$-regularized minimization of the reconstruction error by the algorithm \textsc{Panpal} in Sec.~\ref{sec:panpal}. We formulate the encoding via code tables in the form of a Boolean matrix factorization, which defines together with a suitable relaxation of this measure, derived in Sec. \ref{sec:primp}, the algorithm \textsc{Primp}. In Sec. \ref{sec:experiments}, we compare our approach to related methods in various synthetically generated settings and real-world data. Furthermore, we provide insight into the algorithms' understanding of noise based on images. Finally, we conclude in Sec. \ref{sec:conclusion}.  
\section{Problem Definition and Building Blocks} \label{sec:background}
We identify items $\I=\{1,\ldots,n\}$ and transactions $\T=\{1,\ldots,m\}$ by a set of indices of a binary matrix $D\in \{0,1\}^{m\times n}$. This matrix represents the data, having  $D_{ji}=1$ iff transaction $j$ contains item $i$. A set of items is called a \emph{pattern}. If the pattern is a subset of a transaction, we say the transaction \emph{supports} the pattern.  

Throughout the paper, we often employ the function $\theta_t$ which rounds real to binary values, i.e., $\theta_t(x)=1$ for $x\geq t$  and $\theta_t(x)=0$ otherwise. We abbreviate $\theta_{0.5}$ to $\theta$ and denote with $\theta(X)=(\theta(X_{ji}))_{ji}$ the entry-wise application of $\theta$ to a matrix $X$.

We denote matrix norms as $\|\cdot\|$ for the Frobenius norm and $|\cdot|$ for the entry-wise 1-norm. These norms are equivalent for binary matrices $X$ in the sense that $|X|=\|X\|^2$. We use the short notation $|X|_- = |\theta(-X)|$ and $|X|_+=|\theta(X)|$ to separate the norm of negative and nonnegative entries of $X$.  
We often abbreviate the notation of a matrix $(x_{ij})_{1\leq i\leq n,1\leq j\leq m}$ to $(x_{ij})_{ij}$ if the range of indices is clear from the context. Correspondingly, we notate column vectors $(x_i)_i$. The operator $\circ$ denotes the Hadamard product which multiplies two matrices of same dimensions element-wise. Lastly, we remark that $\log$ denotes the natural logarithm.  
\subsection{Problem Definition}\label{sec:task}
\begin{figure}
\centering
\include{pics/tilingFact}
\caption{An exact Boolean factorization using two tiles. The tiles are highlighted.}
\label{fig:tilingFact}
\end{figure}
We seek sets of column-row selections which can be visualized as a formation of blocks as exemplified in our introduction. The expectation that such a formation exists is based on the assumption that the data $D$ originates from a Boolean matrix product $\theta(YX^T)$. Here, it is important to understand that the thresholding function $\theta$ (defined above) suffices to map binary operations onto the Boolean algebra where the addition corresponds to the logical conjunction, i.e., $\theta(0+1)=\theta(1+0)=1$ and $\theta(1+1)=1$. This way, the product is also well-defined for nonnegative real-valued matrices. For an exploration of non-canonical Boolean matrix products, and what derives from them, see, e.g., \citep{miettinen2015generalized}. 

Let $X$ be an $n\times r$ binary matrix and $Y$ an $m\times r$ binary matrix. The product $\theta(YX^T)$ specifies $r$ column-row selections called tiles, one by each pair of column vectors $(X_{\cdot s},Y_{\cdot s})$. Thereby, we implicitly assume that each tile provides information about co-occurrences of items and transactions. In practice, however, one column vector may be equal to the zero vector or indicate only one item, respectively a single usage of a pattern. We do not take these trivial column-row selections into account and introduce the function $r(\cdot,\cdot)$ to count the number of valuable tiles $r(X,Y)=|\{s:|X_{\cdot s}|>1 \wedge |Y_{\cdot s}|>1\}|\leq r$. In theory, we often assume that $r=r(X,Y)$ and in this case, we call $r$ the rank of the tiling or factorization. An example of a rank-2 factorization is depicted in Fig. \ref{fig:tilingFact}.  The vector $X_{\cdot s}$ indicates the pattern which contains all items $i\in \mathcal{I}$ with $X_{is}=1$. Likewise, the vector $Y_{\cdot s}$ marks the transactions which use pattern $X_{\cdot s}$ to form the tile. Accordingly, we refer to the matrix $X$ as the \emph{pattern matrix} and to  $Y$ as the \emph{usage matrix}. The name \emph{tile} reflects its visualization as a single block matrix $Y_{\cdot s}X_{\cdot s}^T$ for suitably rearranged columns and rows.  

We state the informal problem to recover the latent factorization, which we intend to approach in this paper, as follows. 
\begin{learningProblem}
  \problemtitle{\textbf{Informal Problem Definition}}
  \probleminput{ a data matrix $D\in\{0,1\}^{m\times n}$ originating from the following process
\begin{enumerate}
		\item Let $X\in\{0,1\}^{n\times r}$ be a rank $r$ matrix, denoting $r$ patterns.
        \item For each transaction $D_{j\cdot}$ choose a set of patterns $S_j\subseteq \{1,\ldots r\}$.
        \item Construct $Y\in\{0,1\}^{m\times r}$ such that $Y_{js}=1\Leftrightarrow s\in S_j$.
        \item Let $N\in\{-1,0,1\}^{m\times n}$ be the noise  matrix satisfying 
        \[N_{ji}+\theta(YX^T)_{ji}\geq 0.\]
        \item Set
        \begin{align} \label{eq:decomposeD}
			D=\theta(YX^T)+N.
  		\end{align}
\end{enumerate}
}
\problemquestion{the factor matrices $X$ and $Y$.}
\end{learningProblem}
Solving this problem is infeasible in practice as the data generation process is not invertible. Therefore, an approachable surrogate problem is formulated: the minimization of a function, which shall, together with suitable constraints on the result space, indicate the quality of the derived model. In the following, we inspect how three related research branches formulate and tackle such minimization problems, namely \textit{Tiling}, the \textit{Minimum Description Length} principle and \textit{Nonnegative Matrix Factorization}.    
\subsection{Tiling}\label{sec:Tiling} 
Tiling addresses the task to find binary matrices which minimize a given cost measure in a restricted search space.
Many cost measures have been formulated with respect to tiling. Each one defines different criteria of what makes a set of tiles suitable. We list the most important considerations in Table~\ref{tbl:tiling} according to the following task formalization: 
\begin{learningProblem}
  \problemtitle{\textbf{Tiling}}
  \probleminput{a binary database $D\in\{0,1\}^{m\times n}$, a set of real-valued functions $c\in \mathcal{C}$, a set of natural numbers $\mathcal{R}$ and a cost measure $f$.}
  \problemquestion{a tiling 
  {
  \begin{align*}
	\min_{(X,Y)}&f(X,Y,D)\\
    \text{subject to   }
    	& c(X,Y,D) \leq 0,\ c\in \mathcal{C}\\
        & X\in \{0,1\}^{n\times r},\ Y\in\{0,1\}^{m\times r},\ r\in\mathcal{R}
  \end{align*}
  }
  }
\end{learningProblem}   
\begin{table}
	\centering
	\begin{tabular}{lcp{2.8cm}l}
    	\toprule
        Algorithm & $f(X,Y,D)$ &Constraints $\mathcal{C}$\newline $(N=D-\theta(YX^T))$ & $\mathcal{R}$ \\\midrule \midrule
        \textsc{LTM}
        & $r$&$|N|=0$&$\N$ \\ 
        $r$-\textsc{LTM}
        & $-|\theta(YX^T)|$ & $|N|_-=0$ & $\{r\}$\\ 
         \textsc{Hyper}+
        & $|X|+|Y|$&$|N|_{+}=0, |N|_-\leq \beta$&$\N$ \\ 
        \textsc{MaxEnt}
        & \(\displaystyle-\frac{IC_p(X,Y)}{L_p(X,Y)}\)&$\emptyset$&$\N$\\ 
        \textsc{Asso}
        & $f_{\mathsf{RSS}}(X,Y,D)$ &$\emptyset$ &$\{r\}$\\
        \textsc{Krimp}, \textsc{Slim}, \textsc{SHrimp} 
        &$f_{\mathsf{CT}}(X,Y,D)$&$|N|_-= 0$&$\N$ \\
        \textsc{Groei} 
        & $f_{\mathsf{CT}}(X,Y,D)$&$|N|_-=0$&$\{r\}$ \\
        \textsc{Panda}
        & $f_{\mathsf{L1}}(X,Y,D)$ &$\emptyset$& $\{r\}$\\
        \textsc{Mdl4bmf}, \textsc{Nassau}
        &$ f_{\mathsf{TXD}}(X,Y,D)$&$\emptyset$&$\N_{\leq\min\{n,m\}}$\\ 
        \textsc{Mdl4bmf}, \textsc{Panda+}
        &$ f_{\mathsf{TX}}(X,Y,D)$&$\emptyset$&$\N_{\leq\min\{n,m\}}$\\ 
        \bottomrule
    \end{tabular}
    \caption{Overview of tiling cost measures and implementing algorithms. $\N_{\leq a}$ denotes the set of natural numbers less than or equal to $a$.}
    \label{tbl:tiling}
\end{table}
We see in Table~\ref{tbl:tiling} that many algorithms prohibit negative noise $(|N|_-=0)$. We call such tilings \textit{restrained}, as the usage of patterns is restrained to the supporting transactions.

\cite{tilingGeertz04} consider the cost measure as a measure of interestingness of patterns. In this setting, a tile is determined by its pattern because its usage is identified with all supporting transactions. This implicitly excludes negative noise but enables the application of pattern-mining techniques in the algorithms \textsc{LTM} and r-\textsc{LTM}. 

\cite{kontonasios2010information} and \cite{xiang2011summarizing} argue that the integration of negative noise enables more succinct descriptions and makes the tiling robust to noise. If negative noise is not allowed, every flip of a single bit in the interior of a tile breaks it into two.
\cite{xiang2011summarizing} propose with the greedy algorithm \textsc{Hyper+} to mine restrained tiles first and to combine them to larger (noisy) tiles in a second step, as long as a specified amount of negative noise is not exceeded.  
\cite{kontonasios2010information} propose the information theoretical regulation of noise. The algorithm \textsc{MaxEnt} greedily selects the tile with the highest information ratio among a set of input candidate tiles. The information ratio puts the information content $IC_p(X,Y)$ in relation to the description length $L_p(X,Y)$ of a tiling, given a maximum entropy distribution $p$ over data matrices. 
Both algorithms include negative noise only in a post-processing step and provide no mechanisms to directly derive suitable unrestrained tiles.

\cite{discreteBasisProb}  strive for direct minimization of the approximation error under the umbrella term Boolean Matrix Factorization (BMF). 
They show that the tiling of rank $r$ which yields the minimum error $f_{\mathsf{RSS}}(X,Y,D)=|D-\theta(YX^T)|$ cannot be approximated within any factor in polynomial time (unless $\mathbf{NP}= \mathbf{P}$). Accordingly, they propose a heuristic to solve this problem. \textsc{Asso} incrementally creates $r$ tiles by selecting a pattern first and minimizing the error subject to the usage afterwards. Decisive for the quality of the returned factorization is the choice of the rank $r$. To determine this parameter automatically, several algorithms implement one key paradigm called Minimum Description Length (MDL).
\subsection{The MDL Principle}\label{sec:MDL}
MDL is introduced by \cite{RissanenMdl} as an applicable version of Kolmogorov complexity~\citep{KolmogorovComplexity,mdlGrunwald}.
The learning task to find the best model according to the MDL principle is given as follows: 
\begin{learningProblem}
\problemtitle{\textbf{Description Length Minimization}}
\probleminput{ data $D$ and a set of models  $\mathcal{M}$.}
\problemquestion{ a model $M\in\mathcal{M}$ for $D$ which minimizes the description length
\[ 
	L(D, M) = L^D(D,M) + L^M(M),
\]
where $L^D(D,M)$ denotes the compression size of the database in bits  (using model $M$ for the encoding) and $L^M(M)$ is the description size in bits of the model $M$ itself.}
\end{learningProblem}
Specifications of this task differ in the definition of the encoding which defines the set of models $\mathcal{M}$. Typical models for tilings are given by the factorizations which satisfy the constraints
\[
	\mathcal{M} = \{(X,Y)\in \{0,1\}^{n\times r}\times\{0,1\}^{m\times r}\mid c(X,Y,N) \leq 0\ \forall c\in \mathcal{C}, r\in\mathcal{R}\}.
\]
An encoding which is successfully applied in the area of pattern mining and which we discuss later in the context of tiling, uses code tables  as proposed by \cite{siebes2006item}. 
The code table assigns optimal prefix-free codes to a set of patterns, such that the code lengths can be calculated without realizations of actual codes.
We imagine the code table two-columned: itemsets are listed on the left and assigned codes on the right. Such a dictionary from itemsets to code words can be applied to databases similarly as code words to natural language texts. However, the code usage is not as naturally defined as for words in a text. Patterns are not nicely separated by blanks and the possibilities to disassemble a transaction into patterns are numerous. Therefore, we require for every transaction the indication of its cover by patterns of the code table. This is modeled by a function $cover$, which partitions $D_{j\cdot}$ into patterns of the code table.

Let $CT=\{(\mathit{X_\sigma},C_\sigma)|1\leq \sigma\leq \tau\}$ be a code table of $\tau$ patterns $X_\sigma$ together with their assigned codes $C_\sigma$.
For any distribution $P$ over a finite set $\Omega$, an optimal set of prefix-free codes exists \cite[Theorem 5.4.1]{CoverThomas} such that the number of required bits for the code of $x\in\Omega$ is approximately
\begin{equation*}
	L(code(x)) \approx -\log(P(x)).
\end{equation*}
Desiring that frequently used codes are shorter in size, \cite{siebes2006item} introduce the function $usage$ that maps a pattern to the number of transactions which use it for their cover, i.e.,
\[usage(X_\sigma)=|\{X_\sigma\in cover(CT,D_{j\cdot})\mid j\in\mathcal{T}\}|.\] 
The probability mass function over all itemsets $X_\sigma$ in the code table  is defined as
\begin{align}\label{eq:krimpCodeProb}
P(X_\sigma) = \frac{usage(X_\sigma)}{\sum_{1\leq \rho \leq \tau}usage(X_\rho)}.
\end{align}
This implies that $L(C_\sigma)=-\log P(X_\sigma)$. The data matrix is encoded by a transaction-wise concatenation of codes, denoted by the cover, i.e., transaction $D_{j\cdot}$ is encoded by a concatenation of codes $C_\sigma$ with $X_\sigma\in cover(CT,D_{j\cdot})$. Therefore, code $C_\sigma$ occurs $usage(X_\sigma)$ times in the encoded dataset. The size of the data description is thus computed by
\begin{align*}
	L^D_{CT}(D,CT)
    &=-\sum_{1\leq \sigma\leq \tau} usage(X_\sigma) \cdot \log(P(X_\sigma)).
\end{align*}
The description of the model, the code table, requires the declaration of codes $C_\sigma$ and corresponding patterns $X_\sigma$. Code $C_\sigma$ has a size of $-\log\left(P(X_\sigma)\right)$. A pattern is described by concatenated standard codes of contained items. Standard codes arise from the code table consisting of singleton patterns only, where the usage of singleton $\{i\}$ for $i\in\I$ is equal to the frequency $|D_{\cdot i}|$. In conclusion, the description size of the model is computed as  
\begin{align*}
	L_{\mathsf{CT}}^M(CT)
    &= -\sum_{\substack{1\leq \sigma\leq \tau \\ usage(X_\sigma)>0}}\left(\log\left(P(X_\sigma)\right) +\sum_{i\in X_\sigma}\log\left(\frac{|D_{\cdot i}|}{|D|}\right)\right).
\end{align*}
Note that the function $L_{\mathsf{CT}}$ originally uses the logarithm with base two. We implicitly reformulate this description length by substituting with the natural logarithm. This is equivalent to multiplying the function by a constant which is negligible during minimization. In return, using the natural logarithm will shorten the derivations in Sec.~\ref{sec:primp}.

\cite{siebes2006item} use a heuristic cover function for the algorithm \textsc{Krimp} which employs a specified, static order on patterns. The cover function greedily selects the next pattern in the order which covers items that are not covered yet. This way, covering patterns must neither overlap nor cover more items than stated by the transaction.
\textsc{Krimp} examines an input set of frequent patterns in another static order, adding a candidate pattern to the code table whenever that improves the compression size. Additionally, pruning methods are proposed to correct the selection of patterns in the code table.

\textsc{Slim}~\citep{slim} differs in its candidate generation, which is   dynamically implemented according to an estimated compression gain and dependent on the current code table. This strategy typically improves the compression size, but mainly reduces  the  amount of returned patterns.  
Still, the number of considered candidates is extremely large in comparison to those who are accepted. Time consumption is dominated by computing the usage for each evaluated candidate.
\textsc{SHrimp} \citep{hess2014shrimp} exploits the indexing nature of trees in order to efficiently identify those parts of the database which are affected by an extension of the code table.
\cite{siebes2011structure} restrict with the algorithm \textsc{Groei} the code table to a constant number of patterns. They resort to a heuristic beam search algorithm,
but only for tiny datasets, the beam width parameter can be set to a level allowing a reasonably wide enough exploration of the search space, or else the run time explodes.

All these algorithms follow the heuristic cover definition of \textsc{Krimp} which prohibits negative noise and tile overlap (we state the function $f_{\mathsf{CT}}$ and discuss the specific relationship between the proposed encoding by code tables and tiling in Sec.~\ref{sec:primp}). Although the employment of code tables is originally motivated as a methodology to obtain concise and compressing descriptions of the data, the encoding is quite wordy in comparison to the output of unrestrained tiling algorithms, which we discuss in the following section. Nonetheless, attempts to revoke the restraints of the tiling are not known to us. 
\subsection{Merging MDL and Tiling}\label{sec:MDL_Tiling}
MDL's incorporated trade-off between model complexity and data fit is apt for the determination of the factorization rank. Algorithms which determine the rank according to the MDL principle implement a similar scheme, so far. 
The costs are identified with the description length  and in every iteration, the rank is increased as long as this results in decreasing costs. For every considered rank, a factorization (tiling) method is invoked, which usually extends the result from the former iteration. The performance of this method depends on the choice of factorization and encoding which determines the description length.

\cite{lucchese2010noise} propose an encoding as it is known for sparse data representations, describing a matrix only by the positions of ones. Consequently, the model is described with $L^M((X,Y))=|X|+|Y|$ bits and the data with $L^D(D,(X,Y))=|D-\theta(YX^T)|$ bits, up to a multiplicative constant. The resulting cost function is denoted as $f_{\mathsf{L1}}(X,Y,D)=|D-\theta(YX^T)|+|X|+|Y|$. The algorithm \textsc{Panda} uses a factorization method which adds a tile to the current tiling in a two stage heuristic, comparable to \textsc{Hyper+}. 
 
\cite{Miettinen2014mdl4bmf} argue that the encoding used in \textsc{Panda} is too coarse. They investigate multiple encodings, applying \textsc{Asso} to incrementally increase the factorization rank. Their best-performing encoding is called \emph{Typed XOR DtM} encoding. This is based upon the description of $n$-dimensional binary vectors by number and distribution of ones. 
We refer to  the Typed XOR DtM description length as $f_{\mathsf{TXD}}$ and to the corresponding algorithm as \textsc{Mdl4bmf}. The experimental evaluation suggests that \textsc{Mdl4bmf}'s rank estimation is accurate in a setting with moderate noise, i.e., less than $15\%$, and moderate number of planted tiles, i.e., less than 15. It seems to have a tendency to underfit, as opposed to \textsc{Panda}, which returns sometimes ten times more tiles than planted.

On the other hand, the framework of \textsc{Panda} can be applied with an arbitrary cost measure. \cite{lucchese2014unifying} enhance the algorithm \textsc{Panda} to a faster version \textsc{Panda+} and evaluate the ability to detect a planted tiling in relation to different cost measures and algorithms. In their evaluation of synthetically generated datasets with less than $10\%$ equally distributed noise, \textsc{Panda+} using Typed XOR costs $f_{\mathsf{TX}}$ is outperforming any other choice. The performance is explained with the objective of \textsc{Panda+}'s factorization method, which aims at minimizing the costs, in contrast to \textsc{Asso}, minimizing only the noise.

Another algorithm which tries to incorporate the direct optimization of the MDL-cost measure is \textsc{Nassau}~\citep{nassau15}. Remarking that the formerly proposed algorithms do not reconsider tiles mined at previous iterations, \textsc{Nassau} refines the whole tiling every few steps in relation to the cost measure. Still, the incorporated factorization method minimizes solely the factorization error. The experiments focus on a setting where negative noise is prevalent. In this case, differences to \textsc{Mdl4bmf} are often hard to capture while \textsc{Nassau} typically outperforms \textsc{Panda+}. 
\subsection{Nonnegative Matrix Factorization}
The Boolean factorization of Eq.~(\ref{eq:decomposeD}) has a popular relative called Nonnegative Matrix Factorization (NMF). Given a nonnegative, real valued matrix $D\in\mathbb{R}_+^{m\times n}$ and a rank $r\in\N$, the goal is to recover nonnegative factors $X\in \R_+^{n\times r}$ and $Y\in \R_+^{m\times r}$ such that $YX^T\approx D$. To find the ``correct'' factorization, again, several objective functions and constraints are proposed. Most commonly, the residual sum of squares (RSS) is minimized
\begin{equation}
	\min_{X,Y} F(X,Y) = \frac{1}{2}\left\|D-YX^T\right\|^2.\label{eq:NMF}
\end{equation}
The function $F$ is nonconvex, but convex in either $X$ or $Y$, if the other argument is fixed. That makes it suitable for the \emph{Gauss-Seidel} scheme, also known as \emph{block-coordinate descent} or \emph{alternating least squares}, an alternating minimization along one of the matrices while the other one is fixed. That is, a sequence $(X_k,Y_k)$ is created by
\begin{align}\label{eq:als}
\begin{split}
X_{k+1} &\in \argmin_{X} F(X,Y_k)\\
Y_{k+1} &\in \argmin_{Y} F(X_{k+1},Y).
\end{split}
\end{align}
However, finding a minimum in every iteration is computationally intensive. Thus,
existing algorithms for NMF approximate the scheme of Eq.~(\ref{eq:als}) in several ways~\citep{wang2013nmfReview}. Often, the minimization step is replaced by a single gradient descent update.

NMF is originally introduced by \cite{PaateroPMF} under the name Positive Matrix Factorization. It received much attention since the publication of the easily implementable multiplicative update algorithm by \cite{lee01}. Their intuitive explanation of coherence between  the nonnegativity constraints and the resulting parts-based explanation of the data~\citep{lee1999parts}, emphasizes the interpretability  of the results.

Although initially the difference between NMF and clustering was emphasized \citep{lee1999parts}, further research affirms inherent clustering properties~\citep{li06}. In this context, columns of $X$ equate cluster centroids and corresponding columns of $Y$ indicate cluster membership tendencies. Restricting $Y$ to a binary matrix makes the memberships definite and the orthogonality constraint $Y^TY=I$ enforces unique cluster assignments. This factorization task is actually equivalent to $k$-means \citep{ding05,ding06,bauckhage2015k}. 
If the data matrix is binary, a binary factorization is also desirable, at least to get interpretable results for the cluster centroids~\citep{li05}. In this way, the factorization can be read as a clustering of items, or by using the transposed product, as a clustering of transactions. This is also known under the terms \textit{biclustering, co-clustering} or \textit{subspace clustering}. 

To the best of our knowledge, \cite{zhangApplication} are the only ones approaching the task of biclustering in conjunction with alternating minimization, the standard procedure to solve NMF. 
They propose two methods: the first one uses gradient descent updates with the longest step size preserving nonnegativity of the factor matrices and integrates the penalization of non-binary values into the minimization of the factorization error. As penalizing function, they choose the Mexican hat function $\omega(x) =\frac{1}{2}(x^2-x)^2$. The second method is designed to find the threshold at which nonnegative factor matrices might be rounded best to binary matrices.

Although these methods have several drawbacks (the former lacks a convergence guarantee and the latter applies a costly backtracking linesearch), the results are very promising in comparison to common greedy biclustering algorithms. However, this branch of research is considered to be substantially different from its formulation in Boolean algebra \citep{Miettinen2014mdl4bmf,lucchese2014unifying}. Indeed, the numerical optimization of the binary factorization is not easily adopted for multiplications in Boolean algebra $\theta(YX^T)$; $\theta$ has a point of discontinuity at $0.5$. Equally, all proposed cost measures in Table~\ref{tbl:tiling} are not continuous for real valued matrices with entries in $[0,1]$.
\section{Merging Tiling, MDL, and NMF}\label{sec:PALTiling}
We wish to find a way out of the greedy minimization of tiling cost measures and ask to which extent the theory behind popular NMF optimization methods may be applied to Boolean matrix factorizations. In conclusion, we propose an adaption of the Gauss-Seidel method to minimize a suitable relaxation of tiling cost measures. Similar to the thresholding algorithm of \cite{zhangApplication}, the matrices are rounded according to the actual cost measure afterwards. Moreover, we incorporate the determination of the factorization rank, utilizing that the cost measure may select fewer tiles than offered.

With our ambition to adapt the alternating minimization for Boolean matrix factorization, we face two problems: First, as mentioned above, the use of Boolean algebra induces points of discontinuity. In particular, the gradient of the cost measures does not exist at all points which hinders the application of standard gradient descent methods. Second, many tiling cost measures are not convex in $X$ or $Y$, not even if the other argument is fixed, which is a necessary condition to prove the convergence of the Gauss-Seidel scheme.  

To begin with, we inspect how NMF (Eq.~(\ref{eq:NMF})) and BMF deal with overlapping tiles. This is the crucial point where Boolean algebra diverges from elementary algebra. An illustration of a binary data matrix $D$ consisting of two  overlapping tiles and its approximation by a NMF is shown in the top two equations of Fig.~\ref{fig:overlapFact}. We see that the factors contain values smaller than one at entries which are involved in overlapping parts. With this, overlapping sections are equally well approximated as non-overlapping components. 
The matrices $D_A$ and $D_B$ in Fig.~\ref{fig:overlapFact} show the resulting approximations when the nonnegative factor matrices are rounded to binary matrices. We find that the reconstruction error is largest when the binary matrices are multiplied in elementary algebra (matrix $D_A$ in Fig.~\ref{fig:overlapFact}).
This illustrates how binary matrix factorization penalizes  overlapping patterns, a feature which is desirable in clustering when clusters are not allowed to overlap. Similarly, NMF would return less overlapping factors at a higher factorization rank. In this case, however, the original data matrix is exactly reconstructed by the Boolean product of thresholded factor matrices (matrix $D_B$ in Fig.~\ref{fig:overlapFact}). That is why we consider the minimization of a relaxed cost measure with respect to the elementary algebra whereby the factorization rank is increased stepwise. An evaluation of the actual cost measure in Boolean algebra on the rounded matrices decides whether the rank shall be increased or not.

\begin{figure}
\centering
\include{pics/OverlapFact}
\caption{Approximation of a binary matrix $D$ with two overlapping tiles (top) applying NMF (second from above) and the factorizations resulting from thresholding the factor matrices to binary matrices in elementary algebra (second from below) and Boolean algebra (below). Tiles are highlighted.}
\label{fig:overlapFact}
\end{figure}

This leads us to the second concern, the minimization of a possibly not even partially convex objective.
\cite{bolte2014proximal} extend the application of the Gauss-Seidel scheme to such a larger class of functions with the \emph{Proximal Alternating Linearized Minimization} (PALM). This technique focuses on objective functions which break down into a smooth part $F:\R^{n\times r}\times \R^{m\times r}\times \R^{m\times n}\rightarrow \R$ and a nonsmooth component $\phi:\{X\in\R^{m\times n}|m,n\in\N\}\rightarrow (-\infty, \infty]$
\begin{align}\label{eq:PalmObj}
	F(X,Y,D)+ \phi(X) +\phi(Y).
\end{align}
Thereby, no convexity assumptions are made on $F$ and $\phi$. Furthermore, the function $\phi$ may return $\infty$, which can be used to model restrictions of the search space, e.g., the non-negativity constraint of NMF. 
The method performs an alternating minimization on the linearized objective, substituting $F$ with its first order Taylor approximation. This is achieved by alternating \emph{proximal mappings} from the gradient descent update with respect to $F$, i.e., the following steps are repeated for $1\leq k \leq K$:
\begin{align}\label{eq:PalmIterX}
X_{k+1} &= \prox_{\alpha_k\phi}(X_k-\alpha_k\nabla_XF(X_k,Y_k,D));\\
Y_{k+1} &= \prox_{\beta_k\phi}(Y_k-\beta_k\nabla_YF(X_{k+1},Y_k,D)).\label{eq:PalmIterY}
\end{align}
The proximal mapping of $\phi$, $\prox_\phi:\dom(\phi)\rightarrow \dom(\phi)$\footnote{$\dom(\phi)$ is the domain of $\phi$} is a function which returns a matrix satisfying the following minimization criterion: 
\[\prox_\phi(X) \in \argmin_{X^\star}\left\{\frac{1}{2}\|X-X^\star\|^2+\phi(X^\star)\right\}.\]
Loosely speaking, the proximal mapping gives its argument a little push into a direction which minimizes $\phi$. For a detailed discussion, see, e.g., \citep{parikh2014proximal}. As we can see in Eqs.~(\ref{eq:PalmIterX}) and (\ref{eq:PalmIterY}), the evaluation of this operator is a base operation. Similarly to the alternating minimization in Eq.~(\ref{eq:als}), finding the minimum of the proximal mapping in every iteration by numerical methods is infeasible in practice. Thus, the trick is to use only simple functions $\phi$ for which the proximal mapping can be calculated in a closed form. 

The variables $\alpha_k$ and $\beta_k$ in Eqs.~(\ref{eq:PalmIterX}) and (\ref{eq:PalmIterY}) are the step sizes, which are computed under the assumption that the partial gradients $\nabla_XF$ and $\nabla_YF$ are globally Lipschitz continuous with moduli $M_{\nabla_XF}(Y)$ and $M_{\nabla_YF}(X)$, i.e.,
\[\|\nabla_XF(X_1,Y,D)-\nabla_XF(X_2,Y,D)\|\leq M_{\nabla_XF}(Y)\|X_1-X_2\|\]
for all $X_1,X_2\in \R^{n\times r}$ and similarly for $\nabla_YF$. If $F$ computes the RSS as stated in Eq.~(\ref{eq:NMF}), the Lipschitz moduli are given as
\[M_{\nabla_XF}(Y)=\|YY^T\|, \quad M_{\nabla_YF}(X)=\|XX^T\|.\]
The step sizes are computed as
\[\alpha_k=\frac{1}{\gamma M_{\nabla_XF}(Y_k)}, \quad \beta_k=\frac{1}{\gamma M_{\nabla_YF}(X_{k+1})},\]
where $\gamma$ is a constant larger than one. The parameter $\gamma$ ensures that the step size is indeed smaller than the inverse Lipschitz constant, which is required to guarantee the convergence. Note that the step sizes are antimonotonic to $\gamma$, i.e., if $\gamma=2$ then the step sizes are almost half as small as they could be. 
Assuming that the infimum of $F$ and $\phi$ exists and $\phi$ is proper and lower continuous, PALM generates a nonincreasing sequence of function values which converges to a critical point. 
\subsection{PAL-Tiling: General Framework}
We employ the optimization scheme PALM to minimize a specified relaxation of a tiling cost measure. Adapting to the terminology of Eq.~(\ref{eq:PalmObj}), we assume that for a factor $\mu\geq 0$ and regularizing function $G$ the relaxation has the form 
\begin{align}\label{eq:F_PALTiling}
	F(X,Y,D)= \frac{\mu}{2}\|D-YX^T\|^2 + \frac{1}{2} G(X,Y).
\end{align} 
Here the multiplication by one half refers to the traditional formulation of the residual sum of squares in Eq.~(\ref{eq:NMF}), which shortens the formulation of gradients. 
The regularizing function $G$ is supposed to be real valued and smooth with partial gradients which are Lipschitz-continuous with moduli $M_{\nabla_YG}(X)$ and $M_{\nabla_XG}(Y)$. That is, 
\[\|\nabla_XG(X_1,Y)-\nabla_XG(X_2,Y)\|\leq M_{\nabla_XG}(Y)\|X_1-X_2\|,\]
and similarly for $\nabla_YG$. It follows from the triangle inequality that the Lipschitz moduli of the partial gradients of $F$ are given by the sum
\begin{align*}
M_{\nabla_XF}(Y) &= \mu\|YY^T\| + \frac{1}{2}M_{\nabla_XG}(Y)\\ 
M_{\nabla_YF}(X) &= \mu\|XX^T\| + \frac{1}{2}M_{\nabla_YG}(X).
\end{align*}
We use the function $\phi$, which is integrated into the objective function as stated in Eq.~(\ref{eq:PalmObj}), to limit the matrix entries to the interval $[0,1]$, i.e., $X\in [0,1]^{n\times r}$ and $Y\in [0,1]^{m\times r}$. As discussed by \cite{zhangApplication}, this prevents an imbalance between the factor matrices in which one matrix is very sparse and the other very dense. Apart from that, we wish that the relaxed optimization returns factor matrices which are as close to binary matrices as possible. Therefore, we incorporate penalty terms for non-binary matrix entries in the function $\phi$.
Choosing $\phi$ as a proper and lower semicontinuous function (to be defined in Sec.~\ref{sec:prox}), the objective meets the requirements of PALM to guarantee the convergence to a critical point in a nonincreasing sequence of iterative function values. 
\begin{algorithm}[t]
\caption{Proximal Alternating Linearized Tiling} 
\begin{algorithmic}[1]
  \Function{PAL-Tiling}{$D,\Delta_r,K,T,\gamma=1.00001$}
  	\State $(X_K,Y_K)\gets (\emptyset,\emptyset)$
    \For {$r\in\{\Delta_r,2\Delta_r,3\Delta_r,\ldots\}$}
    	\State $(X_0,Y_0)\gets$\Call{IncreaseRank}{$X_K,Y_K,\Delta_r$} \Comment{Append $\Delta_r$ random columns}
    \For {$k \in \{0,\ldots, K-1\}$}\label{alg:optStart}  
    	\State $\alpha_k^{-1} \gets \gamma M_{\nabla_XF}(Y_k)$ 
        \State $X_{k+1} \gets \prox_{\alpha_k\phi}\left(X_k-\alpha_k\nabla_XF(X_k,Y_k,D)\right)$ 
       	\State $\beta_k^{-1} \gets\gamma M_{\nabla_YF}(X_{k+1})$ 
        \State $Y_{k+1} \gets \prox_{\beta_k\phi}\left(Y_k-\beta_k\nabla_YF(X_{k+1},Y_k,D)\right)$ 
    \EndFor\label{alg:optEnd}
    \State $(X,Y)\gets \argmin\{f(\theta_x(X_K),\theta_y(Y_K))|x,y\in T\}$\Comment{Threshold to binary matrices}\label{alg:fmin}
    \If {$r-r(X,Y)>1$}
    	\State \Return $(X,Y)$
    \EndIf
    \EndFor
  \EndFunction
\end{algorithmic}
\label{alg:primp}
\end{algorithm}

We sketch our method, Proximal Alternating Linearized Tiling (\textsc{PAL-Tiling}), in Algorithm~\ref{alg:primp}. 
A data matrix $D$, rank increment $\Delta_r$, maximum number of iterations $K$, a set of threshold values $T$ and the parameter $\gamma$, having a default value of $\gamma=1.00001$, are the input of this algorithm. 
For every considered rank, we perform the proximal alternating linearized minimization of the relaxed objective (line \ref{alg:optStart}-\ref{alg:optEnd}). After the numerical minimization of the relaxed objective $F$, the matrices $X_K$ and $Y_K$, having entries between zero and one, are rounded to binary matrices $X$ and $Y$ with respect to the actual cost measure $f$ (line \ref{alg:fmin}). If the rounding procedure returns binary matrices which use at least one (non-singleton) pattern less than possible, the current factorization is returned. Otherwise, we increase the rank and add $\Delta_r$ random columns with entries between zero and one to the relaxed solution of the former iteration $(X_K,Y_K)$.

To apply this scheme, we need to define the penalizing function $\phi$, derive its proximal mapping in a closed form and find a suitable cost measure with its smooth relaxed approximation. 
\subsection{Penalizing Non-Binary Values}\label{sec:prox}
While the Mexican hat function can be seen as an $L2$ regularization equivalent penalizer for binary values, here, we choose an $L1$-equivalent form.
Specifically, we choose $\phi(X)=\sum_{i,j}\Lambda(X_{ij})$, which employs the one-dimensional function 
\[
	\Lambda(x) = 
    \begin{cases}
        -|1-2x|+1 &x\in[0,1]\\
        \infty &\text{otherwise}.
    \end{cases}
\]
\begin{figure}
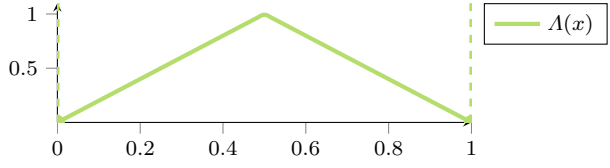

\centering
\include{pics/Lambda}
\caption{The function $\Lambda$ penalizing non-binary values.}
\label{fig:lambda}
\end{figure}
to restrict matrix entries to the interval $[0,1]$ and to penalize non-binary values. The curve of $\Lambda$ is depicted  in Fig.~\ref{fig:lambda}. 
We derive with the following proposition a closed form for the computation of the exact minimum as assigned by the proximal mapping with respect to $\phi$. 
\begin{restatable}{thm}{proxphi}
Let $\alpha>0$ and $\phi(X)=\sum_{i,j}\Lambda(X_{ij})$ for $X\in\R^{m\times n}$. The proximal operator of $\alpha\phi$ maps the matrix $X$ to the matrix $\prox_{\alpha\phi}(X)=A\in [0,1]^{m\times n}$ defined by $A_{ji}=\prox_{\alpha\Lambda}(X_{ji})$, where for $x\in \R$ it holds that
    \begin{equation}\label{eq:prox}
	\prox_{\alpha \Lambda}(x)=
    \begin{cases}
        \max\{0,x-2\alpha\} &x\leq0.5\\
        \min\{1,x+2\alpha\} &x>0.5.
    \end{cases}
    \end{equation}
\end{restatable}
This enables a minimization according to the cost measure of \textsc{Asso}, setting $G(X,Y)=0$. However, without a generalizing term it is unlikely that the rank is properly identified; the loss function certainly attains a minimum when one of the factor matrices is equal to the data matrix and the other one is the identity. Thus, we seek cost measures which are suitable for a minimization within \textsc{PAL-Tiling} and whose application results in an algorithm which is capable to identify the correct rank.  
\subsection{\textsc{Panpal}}\label{sec:panpal}
The cost measure $f_{\mathsf{L1}}$ (as applied by \textsc{Panda}) can easily be  integrated into \textsc{PAL-Tiling}. Since the proximal operator ensures that the factor matrices in all steps have values between zero and one, the $L1$-norm of the factor matrices equates a simple summation over all matrix entries. Thus, the $L1$-norm is a smooth function on the nonnegative domain of the factor matrices and can be used as regularizing function. We call the resulting algorithm \textsc{Panpal} as it employs the cost measure of \textsc{\textbf{Pan}da} in the minimization technique \textsc{\textbf{PAL}-Tiling}:
\begin{mybox}
\textsc{Panpal}: 
Apply \textsc{PAL-Tiling} with
\begin{itemize}
\item Cost measure
	\[f_{\mathsf{L1}}(X,Y,D) = |D-YX^T|+|X|+|Y|\]
\item Relaxed objective
	\[F(X,Y,D)=\frac{1}{2}\|D-YX^T\|^2+ \frac{1}{2}(|X|+|Y|)\]
\item Partial Gradients
\begin{align*}
\nabla_XF(X,Y,D)&=(YX^T-D)^TY+(0.5)_{is}\\
\nabla_YF(X,Y,D)&=(YX^T-D)X+(0.5)_{js}
\end{align*}
\item Lipschitz moduli 
	\begin{align*}
		M_{\nabla_XF}(Y)=\|YY^T\|\\
        M_{\nabla_YF}(X)=\|XX^T\|
\end{align*} 
\end{itemize}
\vspace{2ex}
\end{mybox}
\subsection{\textsc{Primp}}\label{sec:primp}
So far, the cost measure of \textsc{Krimp} has been disregarded in the context of Boolean matrix factorization. Since the traditionally employed cover function is incompatible with overlapping patterns or patterns which cover more items than persistent in the transaction, the task to find the best encoding by code tables is associated with the sub-domain of pattern mining \citep{Miettinen2014mdl4bmf,lucchese2014unifying,nassau15}. The definition of the long-established cover function is heuristically determined under the assumption that there is one globally valid cover function which is applicable on all datasets if a suitable code table is found. Although this approach might be favorable in sub-domains like classification or detection of changes in a data stream \citep{vreeken2011krimp,streamkrimp}, it is current best practice in the domain of tiling to take negative noise into account (see Sec.~\ref{sec:Tiling}).  
Thus, we break away from the conventional view on the cover function as a predefined instance and regard it as an extrapolation of the mapping from patterns to transactions which is defined by the matrix $Y$. Thereby, we intend to learn a suitable pair of code table and cover function for every dataset. This is motivated by the following observation.
\begin{restatable}{lem}{lctbmf}\label{thm:CTBMF}
Let $D$ be a data matrix. For any code table $CT$ and its cover function there exists a Boolean matrix factorization $D=\theta(YX^T)+N$ such that non-singleton patterns in $CT$ are mirrored in $X$ and the cover function is reflected by $Y$. The description lengths correspond to each other, such that 
\[L_{\mathsf{CT}}(D,CT)=f_{\mathsf{CT}}(X,Y,D)=f_{\mathsf{CT}}^D(X,Y,D)+f_{\mathsf{CT}}^M(X,Y,D),\]
where the functions returning the model and the data description size are given as  
\begin{align*}
	f_{\mathsf{CT}}^D(X,Y,D)&=-\sum_{s=1}^r |Y_{\cdot s}| \cdot \log(p_s)
       -\sum_{i=1}^n |N_{\cdot i}| \cdot \log(p_{r+i})\\
       &=L^D_{\mathsf{CT}}(D,CT)\\
    f_{\mathsf{CT}}^M(X,Y,D)
    &=\sum_{s:|Y_{\cdot s}|> 0}\left(X_{\cdot s}^Tc-\log(p_s)\right)
	+\sum_{i:|N_{\cdot i}|> 0}\left(c_i-\log(p_{r+i})\right)\\
    	&=L_{\mathsf{CT}}^M(CT).
\end{align*}
The probabilities $p_s$ and $p_{r+i}$ indicate the relational usage of non-singleton patterns $X_{\cdot s}$ and singletons $\{i\}$,
\[
	p_s = \frac{|Y_{\cdot s}|}{|Y|+|N|},\  p_{r+i} = \frac{|N_{\cdot i}|}{|Y|+|N|}.
\]
We denote with $c\in\R_+^n$ the vector of standard code lengths for each item, i.e., 
\[c_i=-\log\left(\frac{|D_{\cdot i}|}{|D|}\right).\]
\end{restatable}
The proof of this lemma can be found in Appendix \ref{sec:appLCTBMF}.
We remark that this formulation also puts new emphasis on the debate about the model definition in this MDL application. As commented by \cite{siebes2011structure}, the cover function actually is a part of the model and if we learn the cover function together with the code table, an encoding of the data is not possible if only the code table is present. 
However, to be in line with common practice in the field we stick with the description length computation as originally proposed by \cite{siebes2006item}, which is also used in Lemma \ref{thm:CTBMF} and makes our results comparable to previously published results.

The transfer from a code table encoding to a Boolean matrix factorization provides another view on the objective of \textsc{Krimp}-related algorithms. While the focus of matrix factorizations lies on the extraction of a given ground truth, the originally formulated task aims at the derivation of subjectively interesting patterns -- equating interestingness with the ability to compress. Moreover, the tiling derived by Boolean matrix factorizations obliges certain requirements such as the linear independence of columns/rows and the bound on the rank ($r\leq\min\{m,n\}$) which follows from that.

Considering, in reverse, the transfer from a matrix factorization to an encoding by code tables, we naturally receive access to the treatment of negative noise. We can calculate the description size $f_{\mathsf{CT}}$ for arbitrary factor matrices, even if the resulting noise matrix contains negative entries. Yet, the question arises if this also has a suitable interpretation with regard to the encoding. In fact, the interpretation is simple: the items having a negative noise entry can be transmitted just as the items with positive noise entries; their singleton codes are appended to the belonging transaction. If the item is not covered by any other pattern used in this transaction, then it belongs to a positive noise entry and otherwise to a negative one. We obtain therewith a description length equal to $f_{\mathsf{CT}}$. 
That is, each transaction $D_{j\cdot}$ is described by the code concatenation of patterns $X_{\cdot r}$ where $Y_{jr}=1$ and the singleton codes of items $i$ having $N_{ji}\neq 0$.

However, the compression size $f_{\mathsf{CT}}(X,Y,D)$ is not continuous. There are points of discontinuity at encodings which do not use a pattern present in $X$ or one of the singletons. To approximate this function as required in \textsc{PAL-Tiling} (Eq.~(\ref{eq:PalmObj})), we assume that each pattern in $X$ is used at least once. For singletons, we do not wish to make such an assumption; usages of singletons are reflected in the noise matrix and we want to keep the noise as small as possible. Therefore, we bound the description size which is induced by singletons by the RSS. Then, we obtain a smooth function which meets the requirements of \textsc{PAL-Tiling}. This is specified by the following theorem whose proof can be found in Appendix~\ref{sec:appbound}. 
\begin{restatable}{thm}{BoundLCT}\label{thm:bound}
Given binary matrices $X$ and $Y$ and $\mu = 1+\log(n)$, it holds that 
\begin{align} \label{eq:approxLN}
		f^D_{\mathsf{CT}}(X,Y,D) &\leq \mu \|D-YX^T\|^2-\sum_{s=1}^r(|Y_{\cdot s}|+1)\log\left(\frac{|Y_{\cdot s}|+1}{|Y|+r}\right)+|Y|
\end{align} 
\end{restatable}
So far, this bound encompasses the description size of the data, yet the description size of the model is also discontinuous at points where one of the patterns is not used at all. The description size of one side of the code table, representing the patterns by standard singleton codes $c$, computed by 
\[|X^Tc|=\sum_{s=1}^rX_{\cdot s}^Tc\geq\sum_{s:|Y_{\cdot s}|>0}X_{\cdot s}^Tc,\] 
can easily be integrated into the smooth approximation. The matrix $X$ and the standard code sizes $c_i$ contain only nonnegative entries, thus the computation of the $L1$-norm boils down to a summation over all entries in the vector $X^Tc$. 
The remaining terms which compose the model complexity are bounded above by a constant, due to the fixation of the rank during the minimization of the relaxed objective.
Thus, we minimize the relaxed function as denoted in the box below. 
The required Lipschitz constants are computed in Appendix~\ref{sec:applipschitzmoduli}. We refer to this algorithm as \textsc{Primp}, as it performs \textsc{\textbf{P}AL-Tiling} with the objective of \textsc{Kr\textbf{imp}}.   
\begin{mybox}
\textsc{Primp}: Apply \textsc{PAL-Tiling} with
\begin{itemize}
\item Constants
\begin{align*}
c&=(c_i)_{i\in \I}, \ c_i=-\log\left(\frac{|D_{\cdot i}|}{|D|}\right)\\
\mu&= 1+\log(n)
\end{align*}
\item Cost measure
	\[f_{\mathsf{CT}}(X,Y,D)\]
\item Relaxed objective
	\begin{align*}
		F(X,Y,D)&=\frac{\mu}{2}\|D-YX^T\|^2+ \frac{1}{2}G(X,Y)\\
        G(X,Y)&=-\sum_{s=1}^r(|Y_{\cdot s}|+1)\log\left(\frac{|Y_{\cdot s}|+1}{|Y|+r}\right) +|X^Tc| +|Y|
	\end{align*}
\item Partial Gradients
\begin{align*}
\nabla_XF(X,Y,D)&=\mu(YX^T-D)^TY+c(0.5)_s^T\\
\nabla_YF(X,Y,D)&=\mu(YX^T-D)X-\frac{1}{2}\left(\log\left(\frac{|Y_{\cdot s}|+1}{|Y|+r}\right)\right)_{js}+(0.5)_{js}
\end{align*}
\item Lipschitz moduli 
	\begin{align*}
		M_{\nabla_X F}(Y)&=\mu\|YY^T\|\\
        M_{\nabla_Y F}(X)&=\mu\|XX^T\|+m
\end{align*} 
\end{itemize}
\vspace{2ex}
\end{mybox}
\section{Experiments}\label{sec:experiments}
We conduct experiments on a series of synthetic data matrices, exploring the ability to detect the planted tiling, i.e., to recover generated matrices X and Y in presence of various noise structures. In real-world data experiments we compare the costs of obtained models in various measures. Also, we perform a qualitative evaluation of the factor methods, visualizing the algorithms' understanding of tiles and noise on the basis of images. We compare the \textsc{PAL-Tiling} instances \textsc{Panpal} and \textsc{Primp} with the available implementations of 
\textsc{PaNDa+}\footnote{\url{http://hpc.isti.cnr.it/~claudio/web/archives/20131113/index.html}}, \textsc{Mdl4bmf}\footnote{\label{note1}\url{http://people.mpi-inf.mpg.de/~skaraev/}} and \textsc{Nassau}\footnoteref{note1}. Concerning \textsc{Panpal} and \textsc{Primp}, we apply $K=50,000$ iterations and try thresholds with step size $0.05$, i.e., $T=\{0.05k\mid k\in\{0,1,\ldots,20\}\}$. We apply the same set of thresholds $T$ to \textsc{Mdl4bmf}, which is also the average increment used in experiments by \cite{Miettinen2014mdl4bmf,lucchese2014unifying,nassau15}. For \textsc{Panda+} we choose the TypedXOR measure and use 20 randomization rounds and correlating items as suggested in the literature \cite{lucchese2014unifying}. Apart from that, the default settings apply. 

We exclude \textsc{Slim} from our experiments as it can not be fairly compared to Boolean matrix factorization algorithms. To illustrate, \textsc{Slim} returns far more patterns ($500$ to $3000$ monotonically increasing with noise) than planted (25) in our synthetic data sets. Hence, a depiction of this algorithm's rank would distort the rank charts due to its unreasonable performance.

For synthetic and real world experiments, we set the rank increment $\Delta_r=10$; sensitivity to this parameter is explored in Sec.~\ref{sec:expRInc}. For our image evaluation, we set (if possible) a maximum number of $10$ returned tiles and depict the four most informative tiles. Here, we set $\Delta_r=1$, to consistently allow for more factorization rounds in case of higher estimated rank. 

A separate run time comparison of the aforementioned algorithms is not conducted. This is because we can not guarantee that the underlying data structures and platform specific optimizations are equally well tuned, especially for the approaches for which we make use of the reference implementation (\textsc{Panda+}, \textsc{Mdl4bmf} and \textsc{Nassau}). Note, however, that due to the formulation of \textsc{PAL-Tiling} in terms of linear algebra, a highly parallel implementation on graphics processing units (GPU) is straightforward. Therefore, experiments regarding \textsc{Panpal} and \textsc{Primp} are executed on a GPU with 2688 arithmetic cores and 6GiB GDDR5 memory. The run time of the GPU based algorithms is about 50 times lower, compared to the ordinary implementations, e.g., a task that is finished by \textsc{Primp} in a few seconds, requires 30 minutes by \textsc{Mdl4bmf}. We provide the source code of our algorithms together with the data generating script\footnote{\url{http://sfb876.tu-dortmund.de/primp}}.
\subsection{Synthetic Data Generation} 
We generate data matrices according to the scheme established by \cite{Miettinen2014mdl4bmf,nassau15} and \cite{lucchese2014unifying}. Yet, we constrain the set of generated factor matrices to contain at least one percent of uniquely assigned ones to ensure linear independence of column vectors. This ensures that for $r^\star\leq n,m$ and generated matrices $X^\star\in\{0,1\}^{n\times r^\star}$ and $Y^\star\in\{0,1\}^{m\times r^\star}$, the matrix $D=Y^\star {X^\star}^T$ indeed has rank $r^\star$. We describe the data generation process as a function from dimensions $n$ and $m$, rank $r^\star$, density parameter $q$ and noise probabilities $p_+$ and $p_-$. 

\begin{description}
\item [\textbf{GenerateData}($n,m,r^\star,q,p_+,p_-$)]\ 
\begin{enumerate}
\item Set $k= \lceil\frac{n}{100}\rceil$ and $l= \lceil\frac{m}{100}\rceil$. Let $\mathbf{1}_k$ and $\mathbf{1}_l$ denote the $k$- and $l$-dimensional vector filled with ones. Draw factor matrices of the form 
\[
X^\star=\begin{pmatrix}
\mathbf{1}_k &  & \raisebox{-1ex}{\text{\Large 0}}\\
&  \ddots &  \\
\text{\Large 0}&  & \mathbf{1}_k\\
\mid & & \mid\\
x_1 &\cdots & x_{r^\star}\\
\mid & & \mid
\end{pmatrix}, \quad
Y^\star=\begin{pmatrix}
\mathbf{1}_l &  & \raisebox{-1ex}{\text{\Large 0}}\\
&  \ddots & \\
\text{\Large 0}& & \mathbf{1}_l\\
\mid & & \mid\\
y_1 &\cdots & y_{r^\star}\\
\mid & & \mid
\end{pmatrix},
\]
where we draw for $1\leq s\leq r^\star$, $\tilde{n}=n-kr^\star$ and $\tilde{m}=m-lr^\star$
\begin{itemize}
\item $x_s\in\{x\in\{0,1\}^{\tilde{n}}\mid |x|\leq \lfloor q\tilde{n}\rfloor\}$ uniformly random 
\item $y_s\in\{y\in\{0,1\}^{\tilde{m}}\mid  |y|\leq \lfloor q\tilde{m}\rfloor\}$ uniformly random
\end{itemize}
\item Set $D=Y^\star {X^\star}^T+N$ with noise matrix $N$ generated by the following scheme
\begin{itemize}
\item If $D_{ji}=0$ then $N_{ji}=1$ with probability $p_+$
\item If $D_{ji}=1$ then $N_{ji}=-1$ with probability $p_-$
\end{itemize}
\end{enumerate}

\end{description}
We generate datasets for distinct settings with dimensions $(n,m)\in\{(500,1600),\allowbreak(1600,500),\allowbreak(800,1000),\allowbreak(1000,800)\}$, $r\in[5,25]$, $q\in [0.1,0.3]$ and $p_\pm\in [0,25]$. 
Table \ref{tbl:statSynthData} summarizes the basic statistics of the generated datasets.   
\begin{table}
	\centering
	\begin{tabular}{lrrrrrr}\toprule
Variation & $p_+[\%]$ & $p_-[\%]$ & $r$ & $q$ & Density $[\%]$ & Overlap $[\%]$ \\ \midrule
\multirow{2}{*}{Uniform Noise} &
0 & 0 & 25 & 0.1 & $6.6\pm0.8$ & $2.3\pm0.3$\\
& 25 & 25 & 25 & 0.1 & $28.3\pm0.4$ & $2.3\pm0.3$\\ \midrule
\multirow{2}{*}{Pos/Neg Noise} &25 & 3 & 25 & 0.1 & $30.6\pm0.4$ & $3.0\pm0.4$\\
 & 3 & 25 & 25 & 0.1 & $8.5\pm0.4$ & $2.9\pm0.5$\\ \midrule
\multirow{2}{*}{Rank} & 10 & 10 & 5 & 0.1 &  $11.3\pm0.4$ & $0.3\pm0.2$\\
 & 10 & 10 & 45 & 0.1 & $18.6\pm0.9$ & $8.7\pm1.0$\\ \midrule
\multirow{2}{*}{Density} & 10 & 10 & 25 & 0.1 & $15.3\pm0.6$ & $2.3\pm0.3$\\
 & 10 & 10 & 25 & 0.3 & $40.6\pm4.3$ & $26.9\pm6.1$\\\bottomrule
\end{tabular}
\caption{Characteristics of generated datasets. The values are aggregated over eight generated datasets, four for each combination of dimensions $(n,m)\in\{(500,1600),(800,1000)\}$. Overlap denotes the percentage of overlapping entries in relation to the region covered by all tiles together and density is the region covered by all tiles together in relation to $nm$.}\label{tbl:statSynthData}
\end{table}
\subsection{Measuring the Tiling Quality}\label{sec:MeasureExpr}
We quantify how well a computed tiling $(X,Y)$  matches the planted tiling $(X^\star,Y^\star)$ by an adaptation of the micro-averaged F-measure, known from multi-class classification tasks. In this regard, we identify a planted tile $Y^\star_{\cdot s }{X^\star_{\cdot s}}^T$  with a class which contains the tuples $(j,i)$ which indicate ones. Then, a suitable  one-to-one matching $\sigma$ between computed and planted tiles allows to compare the \textit{true} labels $Y^\star_{j s }{X^\star_{i s}}^T$  with the \textit{predicted} labels $Y_{j \sigma(s) }{X_{i \sigma(s)}}^T$.
Therewith, we can naturally calculate precision and recall and finally the $F$-measure. 

We assume w.l.o.g.\@ that $X,X^\star\in\{0,1\}^{n\times r}$ and $Y,Y^\star\in\{0,1\}^{n\times r}$, otherwise we attach zero columns to the matrices such that the dimensions match. 
We compute with the Hungarian algorithm~\citep{hungarian} a permutation $\sigma:\{1,\ldots, r\}\rightarrow \{1,\ldots,r\}$ which matches computed and planted tiles one-to-one such that $\sum_{s=1}^rF_{s,\sigma(s)}$ is maximized. The $F_{s,t}$-measure is calculated for $1\leq s,t\leq r$ by  
\[
	F_{s,t}=2\frac{\pre_{s,t}\cdot\rec_{s,t}}{\pre_{s,t}+\rec_{s,t}},
\]
where $\pre_{s,t}$ and $\rec_{s,t}$ denote precision and recall between the planted tile $(X^\star_{\cdot s},Y^\star_{\cdot s})$ and computed tile $(X_{\cdot t},Y_{\cdot t})$
\begin{align*}
	\pre_{s,t} &= \frac{|(Y^\star_{\cdot s}\circ Y_{\cdot t})(X^\star_{\cdot s}\circ X_{\cdot t})^T|}{|Y_{\cdot t}X_{\cdot t}^T|} &\quad 
    \rec_{s,t} &= \frac{|(Y^\star_{\cdot s}\circ Y_{\cdot t})(X^\star_{\cdot s}\circ X_{\cdot t})^T|}{|Y^\star_{\cdot s}{X^\star_{\cdot s}}^T|}.  
\end{align*}
Having computed the perfect matching $\sigma$, we calculate precision and recall for the obtained factorization by
\begin{align*}
\pre &= \frac{\sum_{s=1}^r|(Y^\star_{\cdot s}\circ Y_{\cdot \sigma(s)})(X^\star_{\cdot s}\circ X_{\cdot \sigma(s)})^T|}{\sum_{s=1}^r|Y_{\cdot s}X_{\cdot s}^T|} 
= \frac{|(Y^\star\circ Y_{\cdot \sigma(\cdot)})(X^\star\circ X_{\cdot \sigma(\cdot)})^T|}{|YX^T|}\\ 
\rec &= \frac{\sum_{s=1}^r|(Y^\star_{\cdot s}\circ Y_{\cdot \sigma(s)})(X^\star_{\cdot s}\circ X_{\cdot \sigma(s)})^T|}{\sum_{s=1}^r|Y_{\cdot s}^\star {X_{\cdot s}^\star}^T|} 
    = \frac{|(Y^\star\circ Y_{\cdot\sigma(\cdot)})(X^\star\circ X_{\cdot \sigma(\cdot)})^T|}{|Y^\star{X^\star}^T|}.
\end{align*}
The micro $F$-measure is defined in terms of precision and recall as defined above. This is equivalent to a convex combination of the $F_{s,\sigma(s)}$-measurements:  
\begin{align*}
	F&=2\frac{\pre\cdot\rec}{\pre+\rec}
    = \sum_{s=1}^r\frac{\left|Y^\star_{\cdot s}{X^\star_{\cdot s}}^T\right| + \left|Y_{\cdot\sigma(s)}X_{\cdot\sigma(s)}^T\right|}{\left|Y^\star {X^\star}^T\right|+\left|Y X^T\right|}F_{s,\sigma(s)}.
\end{align*}
The $F$-measure has values between zero and one. The closer it approaches one, the more accurate the obtained tiling is. The plots which display the $F$-measure indicate the average value with error bars having the length of twice the standard deviation.

We express the values of involved cost measures in relation to the empty model
\[
	\%f(X,Y,D) = \frac{f(X,Y,D)}{f(\mathbf{0}_n,\mathbf{0}_m,D)}\cdot 100.
\]
\subsection{Make some Noise}
In the following series of experiments, varying the noise, we plot the $F$-measure and the rank of the returned tiling against the percentage of noise which is added. The planted factorization has a rank of $r^\star=25$ and density parameter $q=0.1$. The noise level varies from $0\%$ to $25\%$ as displayed on the $x$-axis.  
\begin{figure}
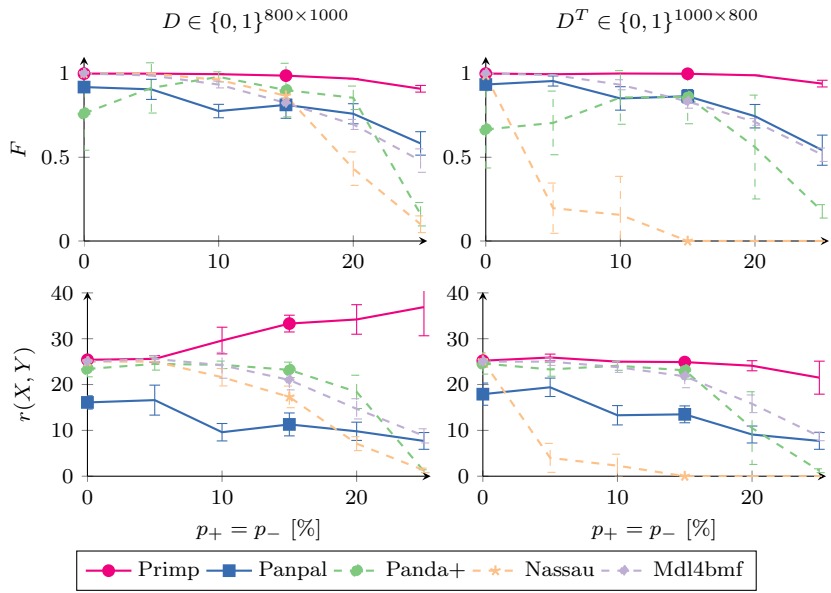

\centering
\include{pics/synthNoise8_10}
\caption{Variation of uniform noise for $800\times 1000$ and $1000\times 800$ dimensional data. Comparison of $F$-measures (the higher the better) and the estimated rank of the calculated tiling (the closer to 25 the better) for varying levels of noise, i.e., $p_+=p_-$ is indicated on the x-axis (best viewed in color).}
\label{fig:noise810}
\end{figure}
\begin{figure}
\centering
\include{pics/synthNoise5_16}
\caption{Variation of uniform noise for $500\times 1600$ and $1600\times 500$ dimensional data. Comparison of $F$-measures (the higher the better) and the estimated rank of the calculated tiling (the closer to 25 the better) for varying levels of noise, i.e., $p_+=p_-$ is indicated on the x-axis (best viewed in color).}
\label{fig:noise516}
\end{figure}

First, we compare the effects of the matrix dimensions and aggregate results over 10 generated matrices with dimensions $800\times 1000$ and $500\times 1600$ together with their transpose, as depicted in Figs.~\ref{fig:noise810} and \ref{fig:noise516}. Comparing the results for a data matrix and its transpose is particularly interesting for the algorithm \textsc{Primp}. Since it applies different regularizations on $X$ and $Y$, we want to asses how this affects the results of \textsc{Primp} in practice. The remaining algorithms minimize an objective which is invariant to a transposition of the input matrix. It is desirable that this is also reflected in practice.

We observe from Figs.~\ref{fig:noise810} and \ref{fig:noise516} that the algorithms likely return fewer tiles the more the noise increases. This culminates in the replication of almost none of the tiles at highest noise level for the algorithms \textsc{Panda+} and \textsc{Nassau}. \textsc{Nassau} particularly strongly underestimates the rank if the data matrix is transposed, i.e., $n>m$. In this case, \textsc{Nassau} returns close or equal to zero tiles, even if the noise is low. \textsc{Panda+} yields correct rank estimations up to a noise of $15\%$, but its fluctuating $F$-measure indicates that planted tiles are not correctly recovered after all. In particular, its $F$-values  differ from the untransposed to the transposed case even if the rank estimations are similar and close to $r^\star$. \textsc{Mdl4bmf} shows a robust behavior towards a transposition of the the input matrix. Its suitable rank estimations up to a noise of $15\%$ are mirrored in a high $F$-measure. \textsc{Panpal} consistently underestimates the rank, yet can achieve comparatively high $F$-measures. Its results exhibit minor deviations from the untransposed to the transposed case. Recognizable differences occur when $n$ and $m$ differ more widely (Fig.~\ref{fig:noise516}) and the noise level is low. Under these circumstances, \textsc{Panpal} yields  higher rank estimations if the matrix is transposed. We note, that  the code of \textsc{PAL-Tiling} and therewith also the code of \textsc{Panpal} inhibits only one distinction between  $X$ and $Y$, which is the order in which gradient steps are invoked. Whether this actually influences the output of the algorithm is an interesting question but it is beyond the scope of this paper.
\textsc{Primp} is characterized by overall high values in the $F$-measure. It has a tendency to estimate the rank higher in the untransposed case, i.e., if $m>n$. This is particularly notably if the  matrices are almost square (Fig.~\ref{fig:noise810}). This suggests that the cost measure favors modeling tiles having fewer items and more transactions. That aside, the overall high $F$-measure shows that additionally modeled tiles cover only a small area in comparison to planted ones. 

\begin{figure}
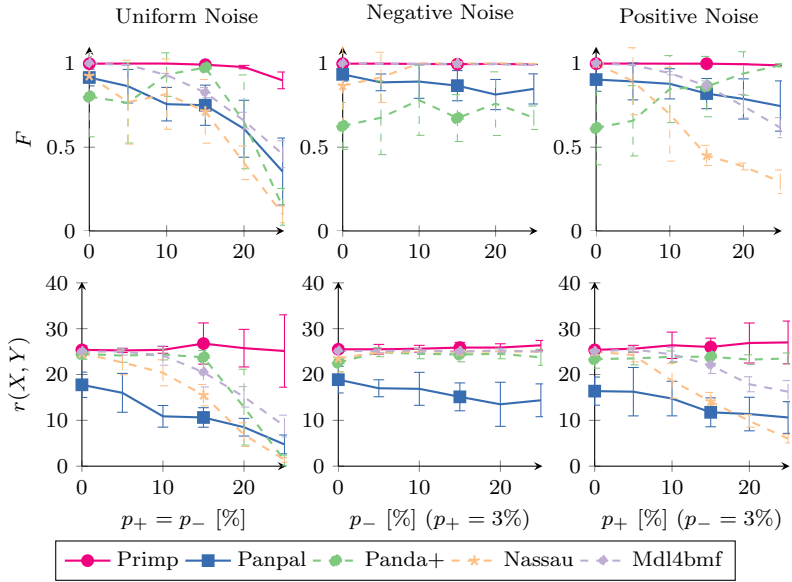

\centering
\include{pics/synthNoise}
\caption{Variation of uniform, positive and negative noise. Comparison of $F$-measures (the higher the better) and the estimated rank of the calculated tiling (the closer to 25 the better) for varying levels of noise, i.e., $p_+$ and $p_-$ are indicated on the x-axis (best viewed in color).}
\label{fig:noise}
\end{figure}
In Fig.~\ref{fig:noise} we contrast varying distributions of positive and negative noise ($p_+$ and $p_-$). From here on, we aggregate results over eight matrices, two for each of the considered matrix dimensions. However, we make an exception for \textsc{Nassau} and transpose the input matrix if $n>m$, as \textsc{Nassau} tends to return zero tiles in this case.

On the left of Fig.~\ref{fig:noise}, we show the aggregated results when varying uniform noise, as discussed for individual dimensions before. All algorithms except for \textsc{Primp} tend to return fewer tiles with increasing noise. Despite correct rank estimations, \textsc{Panda+} displays volatile $F$-measure values. \textsc{Primp}'s rank estimations are correct in the mean, but variance is quite high.

The middle plot depicts variations of negative noise while positive noise is fixed to $3\%$. In this setting, the algorithms \textsc{Primp}, \textsc{Mdl4bmf} and \textsc{Nassau} are capable of identifying the planted tiling for all noise levels. The suitability of \textsc{Nassau} in the prevalence of negative noise corresponds to the experimental evaluation by \cite{nassau15}. \textsc{Mdl4bmf} and \textsc{Primp} yield  equally appropriate results in this experiment.  The approximations of \textsc{Panda+} and \textsc{Panpal} are notably less accurate. Although \textsc{Panda+} correctly estimates the rank around 25 and \textsc{Panpal}'s estimations lie between 10 and 20, \textsc{Panpal} achieves higher $F$-measures than \textsc{Panda+}.  

The plots on the right of Fig.~\ref{fig:noise} show the impact of variations on the positive noise, fixing the negative noise to $3\%$. Here, \textsc{Nassau}, \textsc{Mdl4bmf} and \textsc{Panpal} tend to underestimate the rank the more the noise increases, similarly to but not as drastic as in experiments with uniformly distributed noise. \textsc{Panda+} shows a poor recovery of planted coherent tiles at $0\%$  positive noise, but its $F$-value peculiarly increases with increasing positive noise.  \textsc{Primp} robustly identifies the true tiling for all levels of noise, yet inhibits a higher variance from the mean of the rank estimations.   
\subsection{Variation of Tiling Generation Parameters}
\begin{figure}
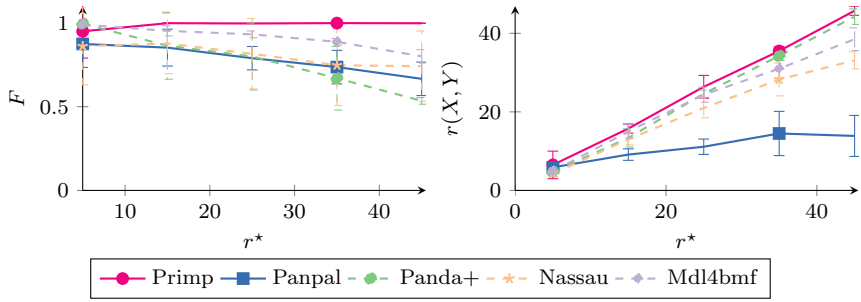

\centering
\include{pics/synthRank}
\caption{Variation of the rank $r^\star\in\{5,\ldots,45\}$ of the planted tiling. Comparison of $F$-measures (the higher the better) and estimated rank (the closer to the identity function the better) of calculated tilings for uniform noise of $p_+=p_-=10\%$ (best viewed in color).}
\label{fig:rank}
\end{figure}
We present effects on variations from the rank in Fig.~\ref{fig:rank} whereby the default parameters of $10\%$ uniform noise and $q=0.1$ apply. We observe a hierarchy of algorithms in the tendency to underestimate the rank throughout all values of $r^\star$. By far the lowest rank estimations are returned by \textsc{Panpal}, followed by \textsc{Nassau}, \textsc{Mdl4bmf}, \textsc{Panda+} and \textsc{Primp}. \textsc{Panda+} and \textsc{Primp} consistently return accurate rank estimations. It is remarkable that for ranks higher than 30, \textsc{Panpal} obtains higher F values than \textsc{Panda+} despite of modeling only a fraction of the planted tiles. \textsc{Primp} provides a steadily accurate recovery of planted tiles.

\begin{figure}
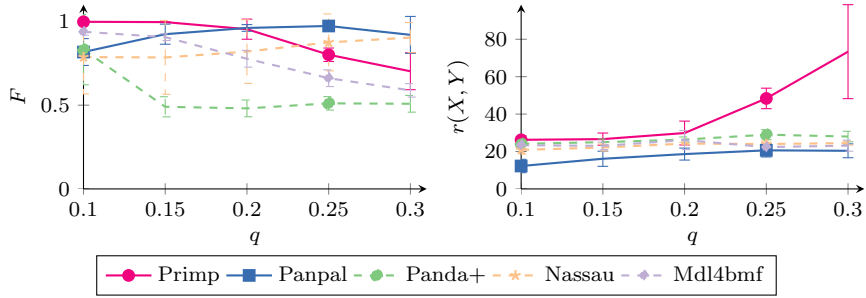

\centering
\include{pics/synthDensity}
\caption{Variation of density and overlap influencing parameter $q\in[0.1,\ldots,0.3]$. Comparison of $F$-measures (the higher the better) and the estimated rank of the calculated tiling (the closer to 25 the better) for uniform noise of $p_+=p_-=10\%$ (best viewed in color).}
\label{fig:density}
\end{figure}
In Fig.~\ref{fig:density} we vary the density and overlap influencing parameter $q$, which determines the maximum density of a column vector in $X$ and  $Y$.  We observe two classes of algorithms. The first class, consisting of \textsc{Primp}, \textsc{Panda+} and \textsc{Mdl4bmf} decreases in the $F$-measure with increasing $q$. In this class, \textsc{Primp} always retrieves highest $F$-values. In return, the $F$-values from the second class of \textsc{Panpal} and \textsc{Nassau} increase with $q$. Here, \textsc{Panpal} bounds the $F$-values of \textsc{Nassau} from above. Correspondingly, \textsc{Primp} and \textsc{Panpal} have a \textit{break-even-point} at $q=0.2$. From this value on, \textsc{Primp} starts to considerably overestimate the rank while \textsc{Panpal}'s tendency to underestimate the rank, decreases. For $q\geq0.2$, \textsc{Panpal} estimates the rank close to 20 in average. That is, five planted tiles are not modeled in average. Still, the $F$-measure indicates that for the denser and more overlapping datasets, \textsc{Panpal} most accurately discovers the planted tiles. 
\subsection{Sensitivity to the Rank Increment}
\label{sec:expRInc}
\begin{figure}
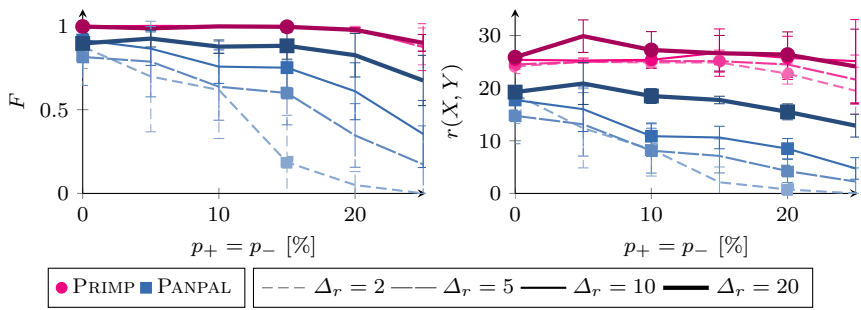

\centering
\include{pics/synthRInc}
\caption{Variation of rank increment $\Delta_r\in\{2,5,10,20\}$. Comparison of $F$-measures (the higher the better) and the estimated rank of the calculated tiling (the closer to 25 the better) for uniform noise of $p_+=p_-$ indicated by the $x$-axis (best viewed in color).}
\label{fig:rInc}
\end{figure}
In the default setting of our synthetic experiments, the \textsc{PAL-Tiling} algorithms \textsc{Primp} and \textsc{Panpal} have to increase the rank two times by $\Delta_r=10$ to estimate the rank $r^\star=25$ correctly. In the experiments varying the rank, we have seen that \textsc{Primp} is able to find the correct rank if twice as many decisions correctly have to be made.  Here, we want to assess how robust the performance of \textsc{Pal-Tiling} algorithms to the parameter $\Delta_r$ is. What happens if, e.g., $\Delta_r=2$ and 23 rank increments have to be administered correctly?

Fig.~\ref{fig:rInc} shows $F$-measurements and estimated ranks of the algorithms \textsc{Primp} and \textsc{Panpal}, invoked with diverse rank increments $\Delta_r\in \{2,5,10,20\}$ on datasets with varying uniform noise. It is noticeable that the rank estimations of \textsc{Panpal} rapidly diverge  with increasing noise while the plots of \textsc{Primp} stay comparatively close. \textsc{Panpal}'s tendency to underestimate the rank grows for smaller rank increments. In return, the rank estimations of \textsc{Panpal} can be improved by choosing a large rank increment, i.e. $\Delta_r\approx r^\star$. However, since we do not know the rank in real world applications, different increment values have to be tried and compared, contradicting our goal to automatically determine this parameter. Still, \textsc{Panpal} yields potentially useful lower bounds on the actual rank.

The average rank estimations of \textsc{Primp} have a maximum aberration of five from the actual rank throughout all noise variations. The graphical display of $r(X,Y)$ for $\Delta_r=20$ has a peak at $5\%$ uniform noise but is close to $r^\star$ otherwise. For rank increments smaller than 10, the estimations do not distinctly decrease until the noise exceeds $20\%$. Here, a rank increment of $\Delta_r=5$ yields the most accurate rank estimations, having also lowest standard deviations from the mean. Particularly, \textsc{Primp}'s tendency to overestimate the rank in specific settings can be corrected by choosing smaller rank increments. Nonetheless, all these rank deviations barely effect the $F$-measure, which demonstrates the robustly well fitted recovery of the underlying model regardless of the choice of rank increment. 
\subsection{Comparison of Cost Measures}
\begin{table}
	\centering
    \begin{adjustbox}{max width=\textwidth}
	\begin{tabular}{clrrrrr}\toprule
 & Algorithm & $\overline{F\strut}$ & $\overline{\%f_{\mathsf{RSS}}\strut}$ & $\overline{\%f_{\mathsf{CT}}\strut}$ & $\overline{\%f_{\mathsf{L1}}\strut}$  & $\overline{\%f_{\mathsf{TXD}}\strut}$  \\ \midrule
 \rowcolor{TUGray!10}
\multirow{6}{*}{\cellcolor{white}\rotatebox{90}{ $p_\pm=25\%$ }  } 
&Planted& 1.0 $\pm$ 0.0 & 88.37 $\pm$ 1.24 & 89.88 $\pm$ 1.25 & 89.51 $\pm$ 1.25 & 96.5 $\pm$ 0.61\\
 & \textsc{Primp} & $\mathbf{0.9\pm0.05}$ & $\mathbf{89.58\pm1.71}$ & $\mathbf{91.0\pm1.54}$ & $\mathbf{90.65\pm1.59}$ & $\mathbf{97.0\pm0.62}$\\
 & \textsc{Panpal} & 0.35 $\pm$ 0.2 & 97.17 $\pm$ 1.62 & 97.56 $\pm$ 1.46 & 97.47 $\pm$ 1.49 & 99.16 $\pm$ 0.56\\
 & \textsc{Mdl4bmf} & 0.46 $\pm$ 0.08 & 96.6 $\pm$ 0.86 & 97.25 $\pm$ 0.76 & 97.11 $\pm$ 0.77 & 99.2 $\pm$ 0.25\\
 & \textsc{Panda} & 0.14 $\pm$ 0.11 & 99.14 $\pm$ 0.77 & 99.28 $\pm$ 0.62 & 99.24 $\pm$ 0.66 & 99.75 $\pm$ 0.19\\
 & \textsc{Nassau} & 0.1 $\pm$ 0.05 & 100.5 $\pm$ 0.29 & 100.7 $\pm$ 0.3 & 100.69 $\pm$ 0.31 & 99.75 $\pm$ 0.15\\
 \midrule
\rowcolor{TUGray!10}
\multirow{6}{*}{\cellcolor{white}\rotatebox{90}{ $r^\star = 45$ }  }  
& Planted & 1.0 $\pm$ 0.0 &  50.27 $\pm$ 1.25 & 54.67 $\pm$ 1.29 & 53.29 $\pm$ 1.29 & 69.77 $\pm$ 0.96\\
 & \textsc{Primp} & $\mathbf{1.0\pm0.0}$ & $\mathbf{50.32\pm1.23}$ & $\mathbf{54.74\pm1.27}$ & $\mathbf{53.35\pm1.27}$ & $\mathbf{69.85\pm0.93}$\\
 & \textsc{Panpal} & 0.67 $\pm$ 0.1 & 73.0 $\pm$ 6.04 & 75.04 $\pm$ 5.56 & 74.29 $\pm$ 5.71 & 84.66 $\pm$ 3.78\\
 & \textsc{Mdl4bmf} & 0.8 $\pm$ 0.04 & 62.67 $\pm$ 1.41 & 66.72 $\pm$ 1.53 & 65.48 $\pm$ 1.46 & 79.21 $\pm$ 1.03\\
 & \textsc{Panda} & 0.53 $\pm$ 0.02 & 89.02 $\pm$ 1.76 & 92.34 $\pm$ 2.16 & 92.76 $\pm$ 2.09 & 86.0 $\pm$ 0.7\\
 & \textsc{Nassau} & 0.74 $\pm$ 0.21 & 64.43 $\pm$ 10.08 & 68.27 $\pm$ 10.24 & 67.28 $\pm$ 10.43 & 77.47 $\pm$ 4.87\\ \midrule
\rowcolor{TUGray!10}
 \multirow{6}{*}{\cellcolor{white}\rotatebox{90}{ $q=0.3$ }  } 
&Planted &  1.0 $\pm$ 0.0  & 24.94 $\pm$ 2.74 & 27.84 $\pm$ 2.78 & 27.11 $\pm$ 2.7 & 51.97 $\pm$ 1.04\\
 & \textsc{Primp} & 0.7 $\pm$ 0.1  & $\mathbf{27.04\pm2.46}$ & $\mathbf{31.23\pm2.33}$ & $\mathbf{30.33\pm2.25}$ & 57.52 $\pm$ 2.03\\
 & \textsc{Panpal}  & $\mathbf{0.92\pm0.11}$ & 29.45 $\pm$ 3.17 & 31.98 $\pm$ 3.06 & 31.31 $\pm$ 3.1 & 57.09 $\pm$ 3.86\\
 & \textsc{Mdl4bmf} & 0.59 $\pm$ 0.04 & 45.08 $\pm$ 2.14 & 48.53 $\pm$ 2.01 & 47.88 $\pm$ 1.96 & 73.81 $\pm$ 1.49\\
 & \textsc{Panda} & 0.51 $\pm$ 0.05 & 54.12 $\pm$ 8.9 & 57.07 $\pm$ 8.83 & 56.74 $\pm$ 8.86 & 75.88 $\pm$ 2.32\\
 & \textsc{Nassau} & 0.9 $\pm$ 0.09 & 29.11 $\pm$ 5.19 & 32.07 $\pm$ 5.2 & 31.42 $\pm$ 5.2 & $\mathbf{56.79\pm4.2}$\\ 
 \bottomrule
\end{tabular}
\end{adjustbox}
\caption{Average cost measures of computed and planted models, denoted relation to the costs of the empty model. For each setting (variation of one data generation parameter while the others are set to default values $r^\star=25$, $q=0.3$ and $p_\pm=25\%$) the average value is computed over all considered dimension variations.}
\label{tbl:avgCosts}
\end{table}

We have seen how well the competing algorithms perform with regard to the $F$-measure. Then again, assessing the performance on real data requires other measurements. Possible candidates are the costs listed in Table~\ref{tbl:tiling}. Subsequently, we relate selected costs of computed and planted models to the $F$-measure and discuss whether we can deduce a suitable extraction of the underlying model from a low cost measure; is smaller always better? 

Table~\ref{tbl:avgCosts} displays the average costs in relation to the empty model for the four measures $f_{\mathsf{RSS}}$, the residual sum of squares, $f_{\mathsf{CT}}$, the compression size obtained by code tables, $f_{\mathsf{L1}}$, the $L1$-regularized residual sum of squares and $f_{\mathsf{TXD}}$, the Typed XOR DtM measure. We examine three parameter settings, one for the highest value in each variation of the data generation parameters $r^\star, q$ and $p_\pm$. Thereby, default settings of $r^\star=25$, $q=0.1$, $p_\pm=10\%$ apply, if not stated otherwise. The values of the planted model are shaded out while the highest $F$-measure and lowest mean costs of computed models are highlighted. 

We can trace that high $F$-values often correspond to lower costs, regardless of the measurement. This effect is immediately perceivable at rows where \textsc{Primp} attains highest $F$-values and all of its cost values are highlighted as well. Yet, the experiments for $q=0.3$ display a  more diverse ranking among the measurements. In this setting, \textsc{Primp} decidedly overestimates the rank but still obtains lowest costs in all but the $f_\mathsf{TXD}$ measure. \textsc{Panpal} attains the highest $F$-value, closely followed by \textsc{Nassau}. Both algorithms reach second or third lowest costs in $f_{\mathsf{RSS}},f_{\mathsf{CT}}$ and $f_{\mathsf{L1}}$. The $f_\mathsf{TXD}$ costs reflect the order of $F$-values more suitably, \textsc{Nassau} obtains lowest costs, closely followed by \textsc{Panpal} and \textsc{Primp}. In brief, the costs of \textsc{Primp}, \textsc{Panpal} and \textsc{Nassau} are always close while only \textsc{Mdl4bmf} and \textsc{Panda+} lie notably behind. Here, the deciding clue is given by the rank, which separates the close cost measurements of \textsc{Primp}, \textsc{Panpal} and \textsc{Nassau} by showing that slight improvements in the costs by \textsc{Primp} are achieved by a disproportionate increase of the rank. 

While for $q=0.3$, the $f_\mathsf{TXD}$ costs appear suitable to reflect an appropriate extraction of tiles, in the setting of $p_\pm=25\%$ we observe another facet. Here, we see that  \textsc{Nassau} reaches the same average $f_\mathsf{TXD}$ costs as \textsc{Panda+} although \textsc{Nassau} increases the RSS in comparison to the empty model. This is indicated by relative costs larger than $100\%$ in all measurements but $f_\mathsf{TXD}$. Still, the ranking of  $f_\mathsf{TXD}$ costs matches the $F$-measure ranking but this example shows, that a compression with respect to the $f_\mathsf{TXD}$ description length can be achieved without adaptation to the data.

\subsection{Real-World Data Experiments}
\begin{table} 
	\centering
	\begin{tabular}{lrrr}\toprule
    Dataset $D$ & $m$ & $n$ & Density $[\%]$\\ 
    \midrule
    Abstracts & 859 & 4977 & 1.02\\
    Mushroom & 8124 & 120 & 19.33\\
    MovieLens5M & 29980 &9044 &1.81\\
    MovieLens500K &3329 & 3015 & 4.99\\
    Chess & 3196 & 75 &49.33\\ \bottomrule
    \end{tabular}
    \caption{Characteristics of considered datasets: Number of rows $m$, number of columns $n$ and density $|D|/(nm)$ in percent.}
    \label{tbl:dataStats}
\end{table}
We conduct experiments on five datasets, whose characteristics are summarized in Table~\ref{tbl:dataStats}. \textit{Chess} and \textit{Mushroom} are discretized benchmark UCI datasets having a comparatively high density and around 50 times more rows than columns. The \textit{Abstracts} dataset indicates the presence of stemmed words, excluding stop-words, in all ICDM paper abstracts until 2007 \citep{DeBie2011}. It is a sparse dataset with around 5 times as many columns (words) as rows (documents). Finally, the \textit{MovieLens5M} and \textit{MovieLens500K} are binarized versions of the \textit{MovieLens10M}\footnote{\url{http://grouplens.org/datasets/movielens/10m/}} and \textit{MovieLens1M}\footnote{\url{http://grouplens.org/datasets/movielens/1m/}} datasets, where rows correspond to users and columns to movies. We set $D_{ji}=1$ iff user $j$ recommends movie $i$ with more than three out of five stars. After selecting only those users which recommend more than 50 movies and those movies which receive more than five recommendations, we obtain two datasets with a balanced number of of rows and columns. The \textit{MovieLens5M} dataset, containing 5M ones, and the \textit{MovieLens500K} dataset, with 500 thousand ones, have a (as one would expect, due to the dataset domain) high amount of negative noise due to missing values. Originally, we intended to consider only the \textit{MovieLens5M} dataset, but \textsc{Nassau} and \textsc{Mdl4bmf} could not terminate in reasonable time -- we aborted the calculations after one month. Therefore, we also prepared the smaller \textit{MovieLens500K} dataset. (For comparison: While \textsc{Primp}, \textsc{Panpal} and \textsc{Panda+} require around ten minutes to compute the result for MovieLens500K, \textsc{Mdl4bmf} and \textsc{Nassau} need more than five days.)  Furthermore, we note that we transpose the \textit{Abstracts} dataset for \textsc{Nassau}, as it returns zero tiles otherwise.

We state the estimated rank and the attained costs, relative to the costs of the empty model for every considered dataset and algorithm in Table~\ref{tbl:realWorldCosts}. The lowest costs are highlighted for each measure and dataset. We observe, similarly to the evaluation in Sec.~\ref{sec:MeasureExpr}, a tendency toward compliance among all measures, except for the Typed XOR description length. 
As such, \textsc{Primp} mostly obtains minimal costs in all datasets but \textit{Mushroom}, where \textsc{Mdl4bmf} reaches lowest costs.   The models of \textsc{Panda+} exhibit for sparse datasets low Typed XOR DtM costs although the fit to the data is low ($f_\mathsf{RSS}>100\%$). The discrepancy between the $f_\mathsf{TXD}$ compression size and the other measurements is most remarkably for the \textit{MovieLens} datasets. Here, the ranking with respect to $f_\mathsf{TXD}$ is  almost inverse to the ranking with respect to other costs.
\begin{table}
	\centering
	\begin{tabular}{cclrrrrr}\toprule
    \multicolumn{2}{c}{Data} & Algorithm & Rank & $ \%f_{\mathsf{RSS}}$ & $ \%f_{\mathsf{CT}}$ & $\%f_{\mathsf{L1}}$  & $\%f_{\mathsf{TXD}}$  \\ \midrule
\multirow{5}{*}{\rotatebox{90}{ Abstracts }  } 
 && \textsc{Primp} & 46 & \textbf{93.0} & \textbf{98.6} & \textbf{96.33} & 96.12\\
 && \textsc{Panpal} & 1 & 99.8 & 99.96 & 99.89 & 99.74\\
 && \textsc{Mdl4bmf} & 24 & 95.84 & 100.68 & 100.49 & 97.05\\ 
 && \textsc{Nassau} & 3 & 99.81 & 101.76 & 103.05 & 96.84\\
 && \textsc{Panda+} & 133 & 113.34 & 125.27 & 140.49 & \textbf{88.19}\\
 \midrule
\multirow{5}{*}{\rotatebox{90}{ Chess }  } &
 & \textsc{Primp} & 18 & \textbf{24.61} & \textbf{31.3} & \textbf{29.32} & \textbf{62.8}\\
 && \textsc{Panpal} & 6 & 40.76 & 46.71 & 45.67 & 78.92\\
 && \textsc{Mdl4bmf} & 3 & 35.91 & 39.34 & 39.51 & 68.88\\ 
 && \textsc{Nassau} & 10 & 31.92 & 39.03 & 38.94 & 65.78\\
 && \textsc{Panda+} & 27 & 25.76 & 36.58 & 35.74 & 65.01\\
 \midrule
\multirow{6}{*}{\rotatebox{90}{MovieLens}} & \multirow{3}{*}{\rotatebox{90}{500K  }  }
 & \textsc{Primp} & 78 & \textbf{88.59} & \textbf{93.29} & \textbf{91.4} & 89.37\\
 && \textsc{Panpal} & 15 & 94.05 & 95.92 & 94.87 & 92.26\\
 && \textsc{Mdl4bmf} & 56 & 89.65 & 94.97 & 93.43 & 88.72\\
 && \textsc{Nassau} & 29 & 111.15 & 118.47 & 120.58 & 85.89\\
 && \textsc{Panda+} & 120 & 160.93 & 165.61 & 168.58 & \textbf{79.49}\\
\cmidrule(lr{0em}){2-8}
 & \multirow{3}{*}{\rotatebox{90}{5M}  }
 & \textsc{Primp} & 209 & \textbf{89.31} & \textbf{93.14} & \textbf{91.2} & 88.16\\
 && \textsc{Panpal} & 38 & 93.68 & 95.72 & 94.39 & 88.73\\
 && \textsc{Panda+} & 1919 & 181.42 & 202.87 & 201.45 & \textbf{72.23}\\
 \midrule
\multirow{5}{*}{\rotatebox{90}{ Mushroom }  }
 && \textsc{Primp} & 14 & 35.75 & 40.89 & 40.25 & 56.09\\
 && \textsc{Panpal} & 7 & 44.03 & 51.23 & 48.52 & 63.75\\
 && \textsc{Mdl4bmf} & 87 & \textbf{23.39} & \textbf{36.6} & \textbf{32.47} & \textbf{50.37}\\
 && \textsc{Nassau} & 65 & 40.60 & 58.77 & 54.93 & 50.62\\ 
 && \textsc{Panda+} & 40 & 100.30 & 117.34 & 112.80 & 66.98\\
 \bottomrule
    \end{tabular}
    \caption{Comparison of cost measures for real-world datasets.}
    \label{tbl:realWorldCosts}
\end{table}
\begin{table}
	\centering
	\begin{tabular}{crrrrr}\toprule
    MovieLens & \textsc{Primp} & \textsc{Panpal} & \textsc{Mdl4bmf} & \textsc{Nassau} & \textsc{Panda+}\\ 
    \midrule
    500K & 2.38 & 2.23 &3.68 & 10.33 & 18.78\\
    5M   & 2.08 & 2.78 &-&-&23.14\\
    \bottomrule
    \end{tabular}
    \caption{Percentage of traceable wrong recommendations of computed models for the MovieLens datasets, i.e., the relative amount of user-movie recommendations which correspond to bad reviews ($<2.5$ stars out of five).}
    \label{tbl:FP}
\end{table}

This leads to the question which cost measure indicates the suitable tiling in such situations? Luckily, we have for the MovieLens data the possibility to assess how many recommendations would fail by the submitted bad reviews, which are not reflected in the input data. We state the relative amount of recommendations which correspond to bad reviews, i.e., $\nicefrac{|D_-\circ \theta(YX^T)|}{|D_-|}$ where $D_-$ is the matrix having $D_{ji}=1$ iff user $j$ rates movie $i$ with less than $2.5$ of five stars, in Table~\ref{tbl:FP}. We observe that the lower $f_\mathsf{TXD}$ costs are, the higher is the rate of recommendation failures, regardless of the estimated rank.  Therefore, we expect \textsc{Primp} to discover the most liable grouping of users and movies, having the lowest approximation error and a very low ration of traceable wrong recommendations. Similarly, it is questionable if low $f_\mathsf{TXD}$ costs indicate suitable models in specific cases where the approximation error diverges such as for the \textit{Abstracts} dataset.
\subsection{Qualitative Inspection of Mined Tiles} 
The $F$-measure gives a hint at the kind of tiling we can expect from the algorithms, e.g., \textsc{Panpal} returns a coarse view, modeling only a few tiles which match actually persistent ones, the quality of \textsc{Panda+}' results substantially varies and \textsc{Mdl4bmf} and particularly \textsc{Primp} are most often able to identify the persistent interrelations. Yet how do the algorithms relate in their actual cognition of structure and noise, what makes a tile a tile?

Image data allows us to visually inspect the resulting factorizations without the need to specify a numeric measure. We can intuitively assess the attempts to capture relevant sub-structures. 
However, some preprocessing is required in order to feed $w \times h$ images to the mining algorithms. 
We employ a standard representation of images: the RGB888 pixel format. Each of the $w \times h$ pixels is represented by $24$ bits, using $8$ bits per color (red, green and blue). 
In order to convert an image into a set of transactions, we divide it into blocks (patches) of $4 \times 4$ pixels, resulting in a total of $\frac{w}{4} \times \frac{h}{4}$ transactions per image. We adopt this representation from computer vision, where image patches are a standard preprocessing step for raw pixel data \citep{Jarrett/etal/2009a}.
Within each block, let $(r,g,b)_{l,k}$ denote the pixel at row $l$ and column $k$, where $r,g,b\in\{0,1\}^8$ are the $8$-bit binary representation of its red, green and blue color values.
We model the concatenation of all $16$ pixels within one block as one transaction
\begin{equation}
\left[(r,g,b)_{1,1},(r,g,b)_{1,2},(r,g,b)_{1,3},(r,g,b)_{1,4},(r,g,b)_{2,1},\dots,(r,g,b)_{4,4}\right]
\end{equation}
which has a length of $24\cdot 16=384$ bits. 

This way, we process two images: an illustration of {\em Alice} in Wonderland (Fig.~\ref{fig:alice}) and a selection of ``aliens'' from the classic game {\em Space Invaders} (Fig.~\ref{fig:spaceInv}). 
We select Alice because the image contains multiple connected areas, each representing a reasonable substructure, i.e., hair, face, dress, arm and background. In return, the Space Invaders image contains multiple patterns in terms of color and shape, but the components are clearly spatially separable. 

\begin{figure}
  \centering
  \begin{tabular}{ccc@{\quad}cccc}
   Original& & Reconstruction & Fac 1 & Fac 2 & Fac 3 & Fac 4 \\
    \includegraphics[width=0.11\columnwidth]{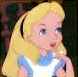}
    & \rotatebox{90}{ \textsc{Nassau} }    
    &  \includegraphics[width=0.11\columnwidth]{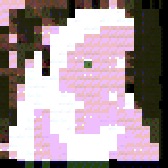} 
    &  \includegraphics[width=0.11\columnwidth]{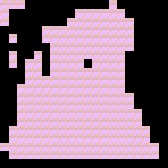}
    &  \includegraphics[width=0.11\columnwidth]{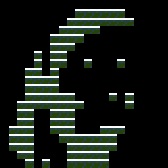}
    & \includegraphics[width=0.11\columnwidth]{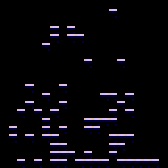}
    & \includegraphics[width=0.11\columnwidth]{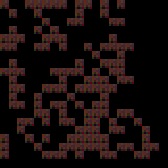} \\ \\
    &\rotatebox{90}{ \textsc{Mdl4bmf} } 
    & \includegraphics[width=0.11\columnwidth]{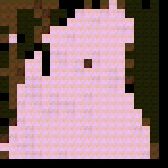}
    & \includegraphics[width=0.11\columnwidth]{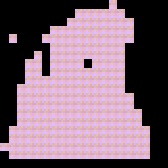}
    & \includegraphics[width=0.11\columnwidth]{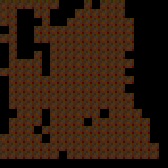}
    &  
    & \\ \\
    & \rotatebox{90}{ \textsc{Panda+} } 
    & \includegraphics[width=0.11\columnwidth]{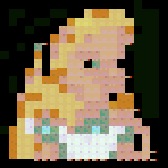}
    &  \includegraphics[width=0.11\columnwidth]{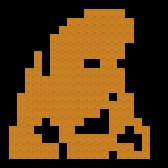}
    & \includegraphics[width=0.11\columnwidth]{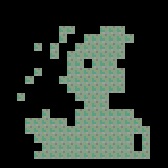}
    & \includegraphics[width=0.11\columnwidth]{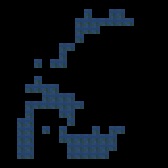}
    & \includegraphics[width=0.11\columnwidth]{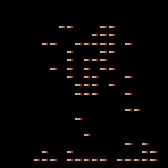}\\ \\
    & \rotatebox{90}{ \textsc{Panpal} } 
    & \includegraphics[width=0.11\columnwidth]{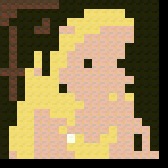} 
    & \includegraphics[width=0.11\columnwidth]{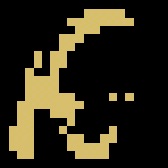}
    & \includegraphics[width=0.11\columnwidth]{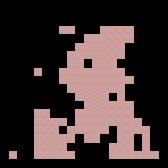}  
    & \includegraphics[width=0.11\columnwidth]{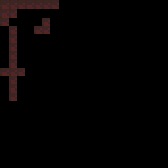} 
    &  \\ \\
    & \rotatebox{90}{ \textsc{Primp} } 
    & \includegraphics[width=0.11\columnwidth]{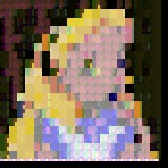}
    & \includegraphics[width=0.11\columnwidth]{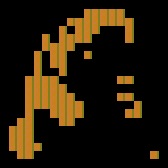}
    & \includegraphics[width=0.11\columnwidth]{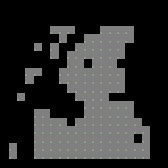}
    & \includegraphics[width=0.11\columnwidth]{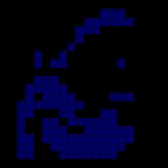}
    & \includegraphics[width=0.11\columnwidth]{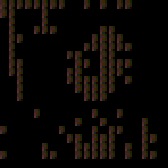}
  \end{tabular}
  \caption{Reconstructions of the Alice image and postprocessed top-4 tiles. Best viewed in color.\label{fig:alice}}
\end{figure}

The original Alice image, as well as reconstructions $\theta(XY)$ and the top-4 tiles generated by \textsc{Nassau}, \textsc{Mdl4bmf}, \textsc{Panda+}, \textsc{Panpal} and \textsc{Primp}, are depicted in Fig.~\ref{fig:alice}. 
Clearly, only \textsc{Panda+} and \textsc{Primp} select patterns, i.e., blocks of pixels which provide a reasonable reconstruction of the original image. \textsc{Panpal}'s tendency to underestimate the rank (choosing only three factors) becomes apparent here again.
Regarding the figured structures, \textsc{Panda+}, \textsc{Panpal} and \textsc{Primp} discover a hair-related substructure, where the one found by \textsc{Primp} has the most distinctive contours, and \textsc{Panda+}, \textsc{Panpal} and \textsc{Primp} identify a face-related structure. The reconstructions and factors found by \textsc{Nassau} and \textsc{Mdl4bmf} are not easy to interpret without knowledge of the original image. 

\begin{figure}
  \centering
  \begin{tabular}{cc@{\quad}cccc}
  	\rotatebox{90}{ Original} & \includegraphics[width=0.14\columnwidth]{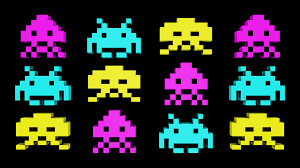} & & & & \\ \\
    & Reconstruction & Fac 1 & Fac 2 & Fac 3 & Fac 4 \\
    \rotatebox{90}{ \textsc{Nassau} }
    &  \includegraphics[width=0.14\columnwidth]{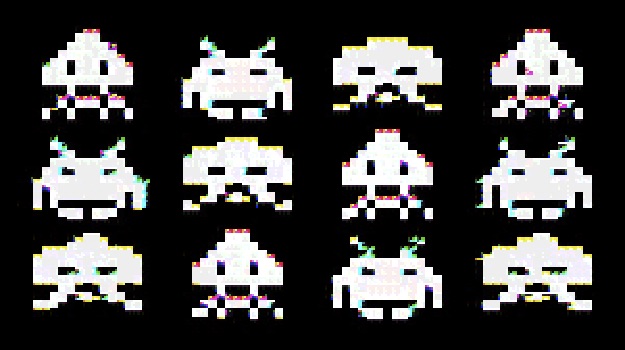}
    &  \includegraphics[width=0.14\columnwidth]{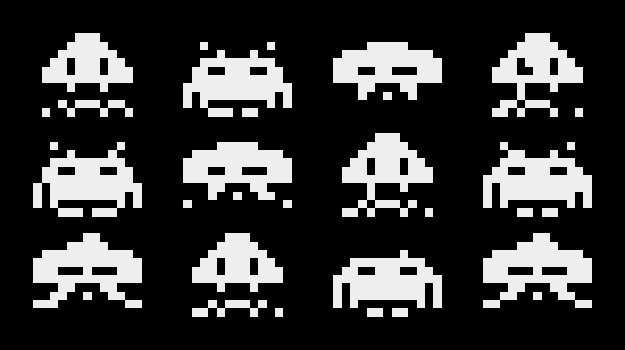}
    &  \includegraphics[width=0.14\columnwidth]{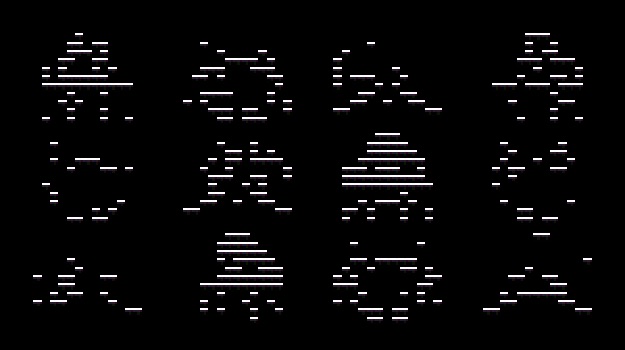}
    & \includegraphics[width=0.14\columnwidth]{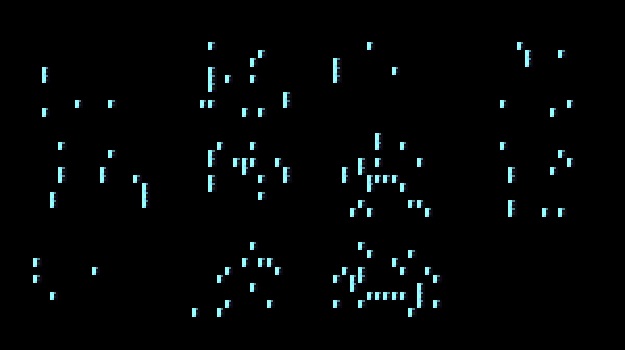}
    & \includegraphics[width=0.14\columnwidth]{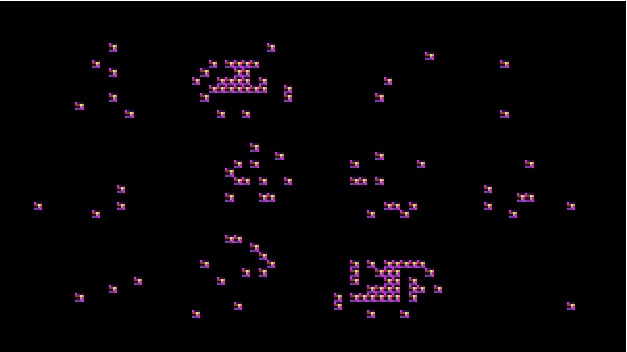} \\ \\
    \rotatebox{90}{ \textsc{Mdl4bmf} } 
    & \includegraphics[width=0.14\columnwidth]{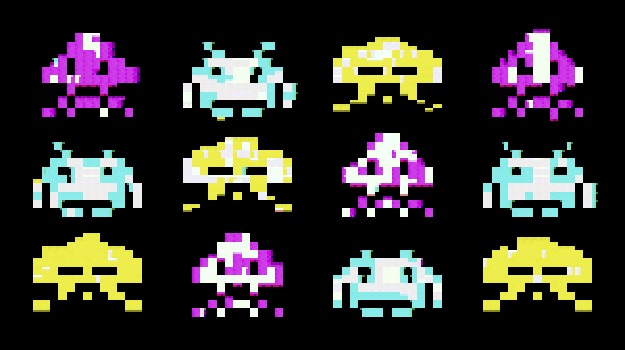}
    & \includegraphics[width=0.14\columnwidth]{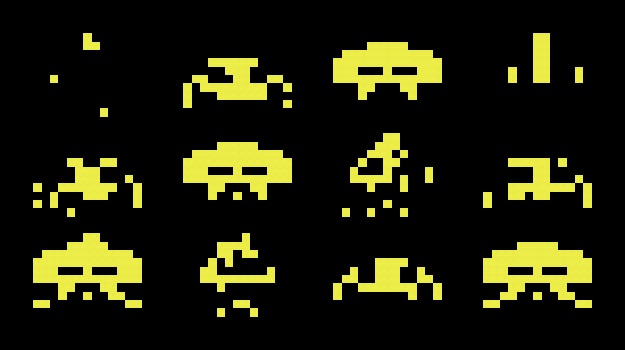}
    & \includegraphics[width=0.14\columnwidth]{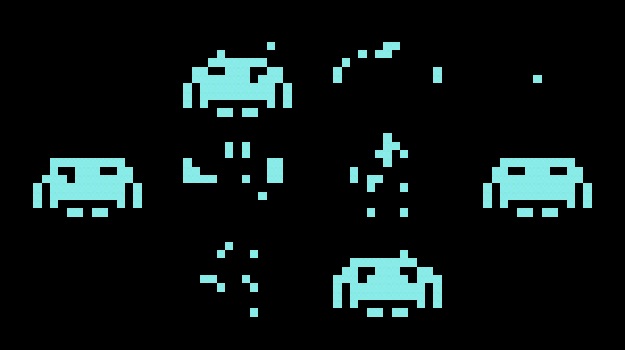}
    & \includegraphics[width=0.14\columnwidth]{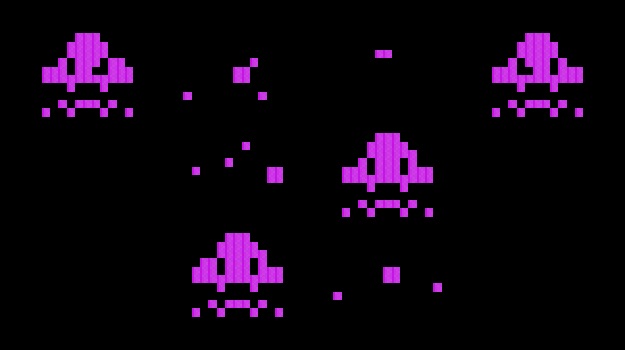} 
    & \includegraphics[width=0.14\columnwidth]{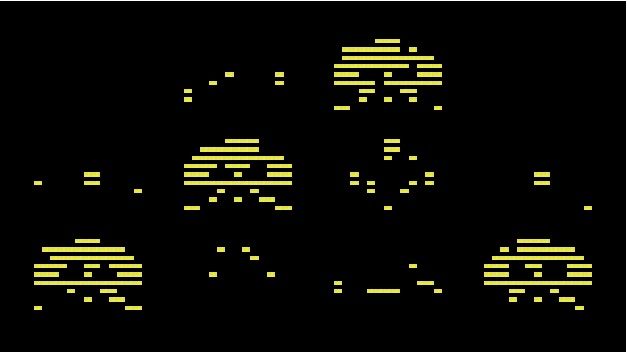}\\ \\
     \rotatebox{90}{ \textsc{Panda+} } 
    & \includegraphics[width=0.14\columnwidth]{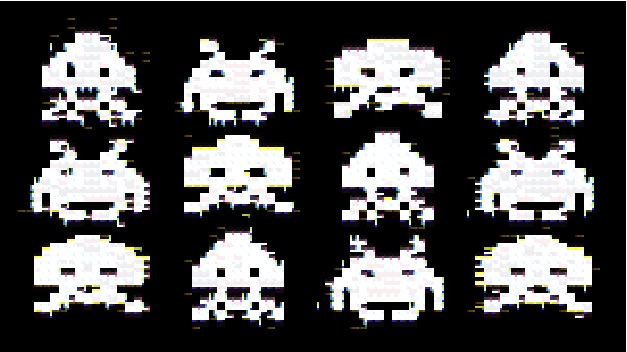}
    &  \includegraphics[width=0.14\columnwidth]{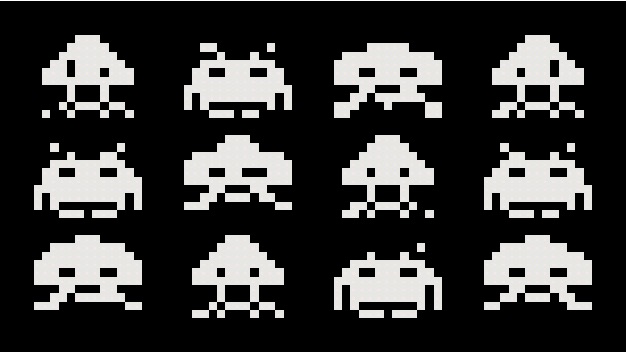}
    & \includegraphics[width=0.14\columnwidth]{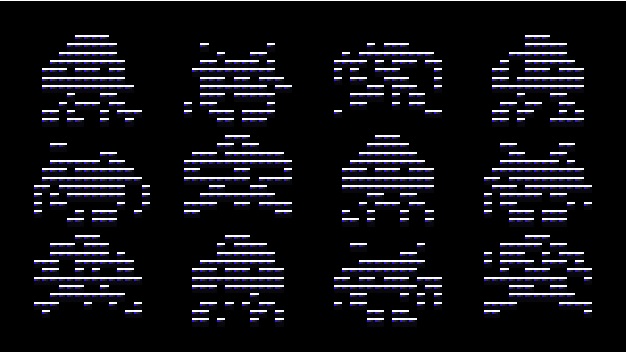}
    & \includegraphics[width=0.14\columnwidth]{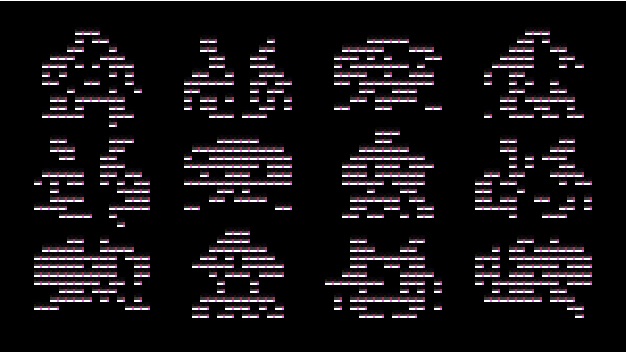}
    & \includegraphics[width=0.14\columnwidth]{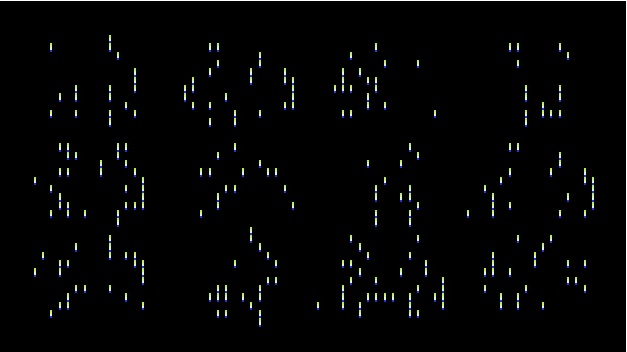}\\ \\
    \rotatebox{90}{ \textsc{Panpal} } 
    & \includegraphics[width=0.14\columnwidth]{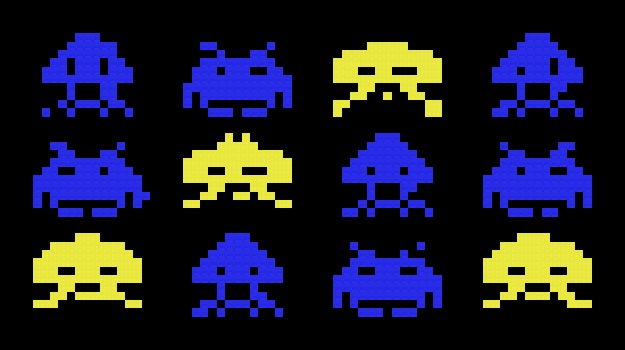} 
    & \includegraphics[width=0.14\columnwidth]{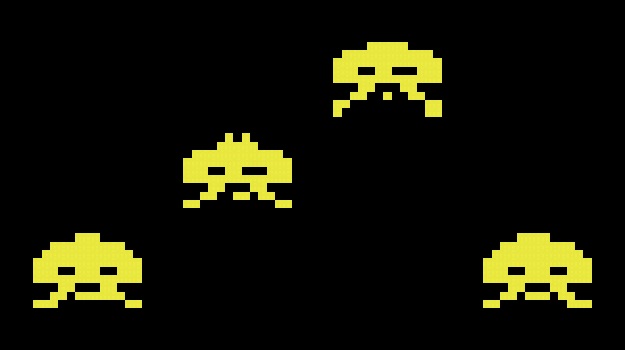}
    & \includegraphics[width=0.14\columnwidth]{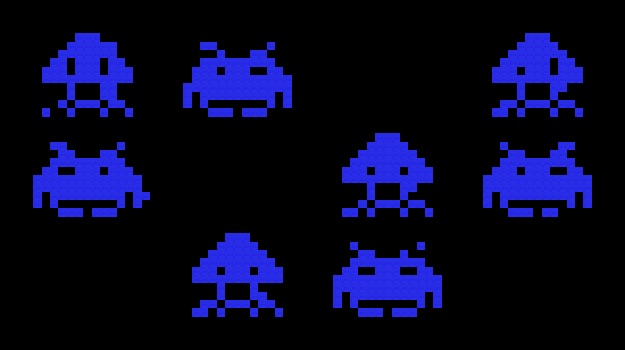}  
    & 
    &  \\ \\
    \rotatebox{90}{ \textsc{Primp} } 
    & \includegraphics[width=0.14\columnwidth]{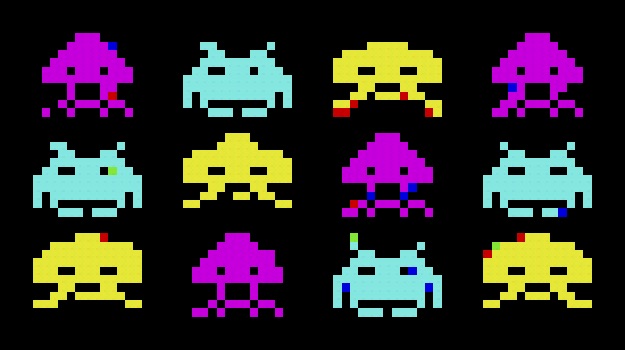}
    &  \includegraphics[width=0.14\columnwidth]{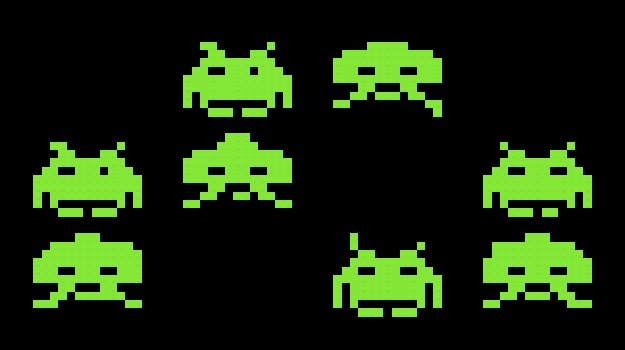}
    & \includegraphics[width=0.14\columnwidth]{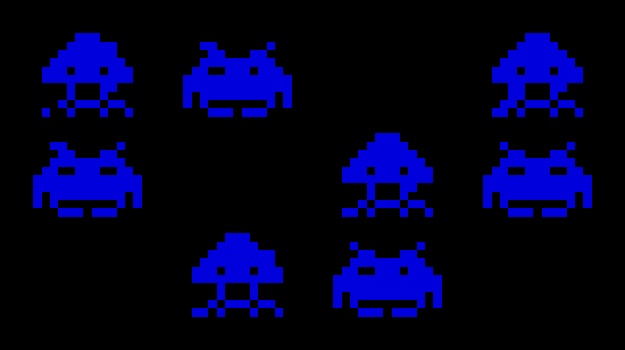}
    & \includegraphics[width=0.14\columnwidth]{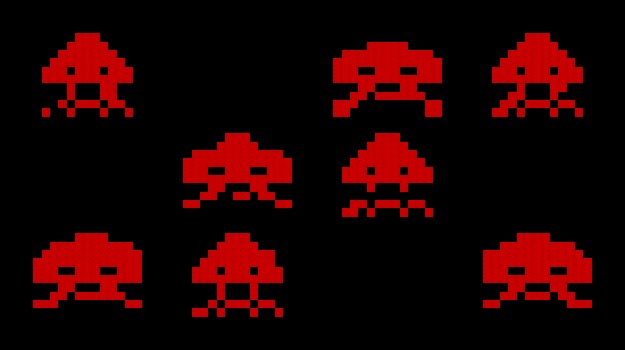}
    & 
  \end{tabular}
  \caption{Reconstructions of the Space Invaders image and the top-4 postprocessed tiles. Best viewed in color.\label{fig:spaceInv}}
\end{figure}
Reconstruction results and top-4 patterns of the Space Invaders image are shown in Fig.~\ref{fig:spaceInv}. All methods reconstruct at least the shape of the aliens. In terms of color, however, the results diverge. \textsc{Panda+} and \textsc{Nassau} interpret all colors as negative noise effects on the color white; white has a binary representation of $24$ ones. \textsc{Panpal} recovers the yellow color correctly and it extracts the full blue channel from the image---an identical pattern is also detected by \textsc{Primp}. \textsc{Primp} and \textsc{Mdl4bmf} reconstruct all three colors of the original image, yet the reconstruction of \textsc{Mdl4bmf} exhibits injections of white blocks. Hence, only \textsc{Primp} is capable to reconstruct the color information correctly. 

Having a look at derived tiles, the greedy processes of \textsc{Panda+} and \textsc{Nassau} become particularly visible; \textsc{Panda+} and \textsc{Nassau} overload the first factor with all the shape information. The remaining factors reduce the quantitative reconstruction error, but have no deeper interpretation. \textsc{Mdl4bmf} tries to model one type of aliens by each tile. Although this would result in a reasonable description of the image, the actual extraction of tiles suffers from the greedy implementation. We can see that, e.g., the first tile captures information about the yellow aliens as well as strayed parts of other aliens. This unfortunate allocation of tiles results in the injection of white blocks in the reconstruction image.  \textsc{Panpal} clearly separates yellow and blue aliens but interprets differences from the color blue to purple and to turquoise as noise. Finally, \textsc{Primp} separates by its tiles the three basic color channels which are actually used to mix the colors that appear in the original image. Hence, \textsc{Primp} achieves the factorization rank that corresponds to the natural amount of color concepts in the image, unlike all other competitors.

The results of this qualitive experiment particularly illustrates the benefits of a non-greedy minimization procedure. Even though \textsc{Panpal} is often not able to minimize the costs due to an underestimation of the rank, its categorization into tiles always yields interpretable parts.
\section{Conclusion}\label{sec:conclusion}
We introduce \textsc{PAL-Tiling}, a general framework to compute tilings according to a cost measure based on a theoretically founded numerical optimization technique. Requiring that the cost measure has a smooth relaxed function,  which combines the matrix factorization error with a regularizing function, \textsc{PAL-Tiling} minimizes the relaxed objective under convergence guarantees. To simulate the minimization subject to the constraint that the matrices are binary, we derive a closed form of the proximal mapping with respect to a function which penalizes non-binary values. A thresholding to binary values according to the actual cost measure enables an automatic determination of the factorization rank.

Aiming at the robust identification of tilings in presence of various noise distributions, we consider two cost measures in this framework which defines two tiling algorithms. The first algorithm uses a simple $L1$-norm regularization on the factor matrices and is called \textsc{Panpal}. The second minimizes the MDL-description length of the encoding by code tables as known from \textsc{Krimp} \citep{siebes2006item}. Foregoing the heuristics in computing the usage of codes, we extend the application of this encoding from pattern mining to Boolean matrix factorization and derive an upper bound which induces the relaxed objective. We refer to this instance of \textsc{PAL-Tiling} as \textsc{Primp}.

Our experiments on synthetically generated datasets show  that the quality of competing algorithms \textsc{Panda+}, \textsc{Mdl4bmf} and \textsc{Nassau} is sensitive towards multiple data generation parameters. The first of the two newly introduced algorithms, \textsc{Panpal}, regularly underestimates the true factorization rank. We have seen that this property can be beneficial in settings with large, overlapping tiles which induce dense datasets (cf.\@ Fig.\@ \ref{fig:density}). In all other settings, the second algorithm \textsc{Primp} is able to detect the underlying structure, regardless of the considered distribution of noise or variations the factorization rank (cf.\@ Figs.\@ \ref{fig:noise810}-\ref{fig:rank}). 

A comparison of cost measures on real-world datasets show that \textsc{Primp} also most often achieves lowest costs (cf.\@ Table\@ \ref{tbl:realWorldCosts}). With experiments based on images, we visualize the derived tiles under  presence of ambiguous tiling structures and special noise distributions (cf.\@ Figs.\@ \ref{fig:alice} and \ref{fig:spaceInv}). The quality of the reconstruction by established algorithms varies considerably between both images. On the contrary, \textsc{Panpal} and \textsc{Primp} provide solid representations of the original images. The extracted factors reveal a parts-based decomposition of the data (as known from non-negative matrix factorizations), which allows for interpretation of the results. In the Space Invaders image (cf.\@ Fig.\@ \ref{fig:spaceInv}), \textsc{Panpal} partitions the space invaders into those with a non-zero blue component in their color (rank-1 factorization 2) and those with a zero blue component in their color (rank-1 factorization 1). On the other hand, \textsc{Primp} divides the space invaders by the primary colors they contain (repeating each space invader exactly twice, hence finding structure in the data too, albeit a different structure from the one found by \textsc{Panpal}). From the Alice image (cf.\@ Fig.\@ \ref{fig:alice}) particularly \textsc{Primp} manages to extract coherent factors representing the hair (rank-1 factorization 1) and the face (rank-1 factorization 4).

The implementation of other popular cost measures , e.g., the Typed XOR DtM, is possible in \textsc{PAL-Tiling} and a topic of future research. Furthermore, the application of other penalizing functions $\phi$ is possible if the corresponding $\prox$-operator can be derived. An analysis of the synergy between the penalizing function, the cost-measure and the thereby derived Boolean Matrix Factorization has the potential to show how the structure from arbitrary binary datasets can be robustly identified.

\begin{acknowledgements}
Part of the work on this paper has been supported by Deutsche Forschungsgemeinschaft (DFG) within the Collaborative Research Center SFB 876 ``Providing Information by Resource-Constrained Analysis'', projects A1 and C1
\url{http://sfb876.tu-dortmund.de}.

Furthermore, we thank Jilles Vreeken and Sanjar Karaev for their support in the execution of experiments and useful remarks.
\end{acknowledgements}


\input{ms.bbl}
\appendix
\section[tbl]{Derivation of the Proximal Operator} \label{sec:appproxproof}
\proxphi*
\begin{proof}
Let $\alpha>0$, $X\in\R^{m\times n}$ for some $m,n\in\N$ and $A=\prox_{\alpha\phi}(X)$. The function $\phi$ is fully separable across all matrix entries. In this case, the proximal operator can be applied entry-wise to the composing scalar functions \citep{parikh2014proximal}, i.e., $A_{ji}=\prox_{\alpha\Lambda}(X_{ji})$. It remains to derive the proximal mapping of $\Lambda$ (Eq.~(\ref{eq:prox})).

The proximal operator reduces to Euclidean projection if the argument lies outside of the function's domain \citep{parikh2014proximal} and it follows that
\[\prox_{\alpha\Lambda}(x)=\theta(x) \text{ if } x\notin[0,1].\]
For $x\in[0,1]$ holds $\Lambda(x)=-|1-2x|+1$ and
\begin{align*}
  \prox_{\alpha\Lambda}(x) &= \argmin_{x^\star\in\R} \left\{\frac{1}{2}(x-x^\star)^2-\alpha|1-2x^\star| +1\alpha\right\}\\
  &= \argmin_{x^\star\in\R} \left\{\underbrace{(x-x^\star)^2-2\alpha|1-2x^\star| +(2\alpha)^2}_{=g(x^\star;x,\alpha)}\right\},
\end{align*}
where $g$ is derived by a multiplication and addition of constants, such that the minimum can easily be derived by completing the square.
\begin{align*}
  g(x^\star;x,\alpha) &=\begin{cases}
    (x-x^\star)^2  -2\alpha(1-2 x^\star) +(2\alpha)^2 & x^\star \leq 0.5\\
    (x-x^\star)^2 +2\alpha(1-2 x^\star) +(2\alpha)^2 & x^\star> 0.5
  \end{cases}\\
  &=
  \begin{cases}
    (x^\star-(x-2\alpha))^2 -2\alpha( 1-2x)& x^\star \leq 0.5\\
    (x^\star-(x+2\alpha))^2 +2\alpha( 1-2x) & x^\star> 0.5
  \end{cases}.
\end{align*}
The function $g$ is a continuous piecewise quadratic function which attains its global minimum at the minimum of one of the two quadratic functions, i.e.,
\[
  \argmin_{x^\star\in\R}g(x^\star;x,\alpha) \in \{x-2\alpha\mid x\leq 0.5+2\alpha\}\cup \{x+2\alpha\mid x> 0.5-2\alpha\}.
\] 
A function value comparison in the intersecting domain $x\in(0.5-2\alpha,0.5+2\alpha]$ yields that
\begin{align*}
g(x-2\alpha;x,\alpha)=-2\alpha(1-2x)\leq g(x+2\alpha;x,\alpha) =2\alpha(1-2x) \Leftrightarrow x\leq 0.5
\end{align*}
\qed
\end{proof}
\section[tbl]{\textsc{Krimp}'s Encoding as Matrix Factorization} \label{sec:appLCTBMF}
\lctbmf*
\begin{proof}
Let $D$ be a data matrix, $CT=\{(\mathit{X_\sigma,C_\sigma})|1\leq\sigma\leq\tau\}$ a $\tau$-element code table and $cover$ the cover function. Let $r$ be the number of non-singleton patterns in $CT$ and assume w.l.o.g. that $CT$ is indexed such that these non-singleton patterns have an index $1\leq\sigma\leq r$. We construct the pattern matrix $X\in \{0,1\}^{n\times r}$ and usage matrix $Y\in \{0,1\}^{m\times r}$ such that for $1\leq\sigma\leq r$ it holds that
\begin{align*}
X_{i\sigma}=1&\Leftrightarrow i\in \mathit{X_\sigma}\\
Y_{j\sigma}=1&\Leftrightarrow \mathit{X_\sigma}\in cover(CT,D_{j\cdot}).
\end{align*}
The Boolean product $\theta(YX^T)$ indicates the entries of $D$ which are covered by non-singleton patterns of $CT$. That implies that ones in the noise matrix $N=D-\theta(YX^T)$ are covered by singletons, it holds that 
\[N_{ji}\neq 0\Leftrightarrow {i}\in cover(CT,D_{j\cdot}).\]
The usage of a non-singleton pattern $X_\sigma$ is then computed as
\begin{align*}
usage(X_\sigma)&=|\{X_\sigma\in cover(CT,D_{j\cdot})|j\in\mathcal{T}\}|\\
&=|\{Y_{j\sigma}=1|j\in\mathcal{T}\}|\\
&=|Y_{\cdot\sigma}|,
\end{align*}
and correspondingly it follows that $usage(\{i\})=|N_{\cdot i}|$. The calculation of the probabilities $p_\sigma$ for $1\leq \sigma \leq r+n$ is directly obtained by inserting this usage calculation in the definition of code-usage-probabilities of Eq.~(\ref{eq:krimpCodeProb}). Likewise follow the functions $f_{\mathsf{CT}}^M$ and $f_{\mathsf{CT}}^D$ from the definition of the description sizes $L_{\mathsf{CT}}^M$ and $L_{\mathsf{CT}}^D$.
\qed
\end{proof}
\section[tbl]{Bounding the Description Length of Code Tables} \label{sec:appbound}
\begin{lemma}\label{thm:monoticity}
Let $(a_s)$ be a finite sequence of $r$ non-negative scalars such that $S_r=\sum_{s=1}^ra_s>0$. The function $g:[0,\infty)\rightarrow[0,\infty)$ defined by
\[g(x;a_1,\ldots,a_r,S_r)=-\sum_{s=1}^r(a_s+x)\log\left(\frac{a_s+x}{S_r+rx}\right)\]
is monotonically increasing in $x$.
\end{lemma}
\begin{proof}
W.l.o.g., let $a_1,\ldots,a_{r_0}>0$ and $a_{r_0+1},\ldots,a_r=0$ for some $r_0\in \N$. We rewrite the function $g$ as
\[g(x;a_1,\ldots,a_r,S_r)=g(x;a_1,\ldots,a_{r_0},S_r)+g(x;a_{r_0+1},\ldots,a_r,S_r)\]
and show that each of  the subfunctions is monotonically increasing.
The first subfunction is differentiable and its derivative is non-negative
\begin{align*}
\frac{d}{dx}g(x;a_1,\ldots,a_{r_0},S_r) &= -\sum_{s=1}^r\left(\log\left(\frac{a_s+x}{S_r+rx}\right)+(a_s+x)\frac{S_r+rx}{a_s+x}\frac{S_r+rx-r(a_s+x)}{(S_r+rx)^2}\right)\\
&= -\sum_{s=1}^r\log\left(\frac{a_s+x}{S_r+rx}\right)+\sum_{s=1}^r\frac{S_r-ra_s}{S_r+rx}\\
&= -\sum_{s=1}^r\log\left(\frac{a_s+x}{S_r+rx}\right)\geq 0.
\end{align*}
The second subfunction is monotonically increasing, since for $a_s=0$ and all $x\geq 0$ it holds that
\begin{align*}
	a_s\log\left(\frac{a_s}{S_r}\right)=0\leq -(a_s+x)\log\left(\frac{a_s+x}{S_r+rx}\right).
\end{align*}
\qed
\end{proof}
\BoundLCT*
\begin{proof}
We recall that the description size of the data is computed by
\[f_{\mathsf{CT}}^D(X,Y,D)=\underbrace{-\sum_{s=1}^r |Y_{\cdot s}| \cdot \log\left(\frac{|Y_{\cdot s}|}{|Y|+|N|}\right)}_{=f_1(X,Y,D)}
       \underbrace{-\sum_{i=1}^n |N_{\cdot i}| \cdot \log\left(\frac{|N_{\cdot i}|}{|N|+|Y|}\right)}_{=f_2(X,Y,D)}.
\]
Applying the logarithmic properties, we rewrite the first sum 
\begin{align*}
 f_1(X,Y,D)&= -\sum_{s=1}^r|Y_{\cdot s}|\log\left(\frac{|Y_{\cdot s}|}{|Y|}\frac{|Y|}{|Y|+|N|}\right)\\
 &= -\sum_{s=1}^r|Y_{\cdot s}|\log\left(\frac{|Y_{\cdot s}|}{|Y|}\right)+\sum_{s=1}^r|Y_{\cdot s}|\log\left(\frac{|Y|+|N|}{|Y|}\right)\\
 &= g(0;|Y_{\cdot 1}|,\ldots,|Y_{\cdot r}|,|Y|) +|Y|\log\left(1+\frac{|N|}{|Y|}\right). 
\end{align*}
It follows from the monotonicity of $g$ (Lemma~\ref{thm:monoticity}) and the logarithm inequality ($\log(1+x)\leq x, \forall x\geq 0$) that $f_1$ is upper bounded by
\[f_1(X,Y,D)\leq-\sum_{s=1}^r(|Y_{\cdot s}|+1)\log\left(\frac{|Y_{\cdot s}|+1}{|Y|+r}\right)+|N|.\]
The second term $f_2$ can be transformed into
\begin{align*}
    f_2(X,Y,D)&=-\sum_{i=1}^n |N_{\cdot i}| \cdot \log\left(|N_{\cdot i}|\right)+\sum_{i=1}^n |N_{\cdot i}| \cdot \log\left(|N|+|Y|\right)\\
    &= \sum_{i=1}^n|N_i|\log\frac{1}{|N_i|} +|N|\log(|N|+|Y|).
\end{align*}
Subsequently, we show $f_2(X,Y,D)\leq |N|\log(n) +|Y|$. This inequality trivially holds if $|N|=0$. Otherwise, we apply Jensen's inequality to the concave logarithm function
	\[|N|\sum_{i=1}^n\frac{|N_i|}{|N|}\log\frac{1}{|N_i|}\leq |N|\log\left(\frac{n}{|N|}\right).\]
and obtain 
\begin{align*}
    f_2(X,Y,D)&\leq|N|\log\left(\frac{n}{|N|}\right) +|N|\log(|N|+|Y|)\\ &= |N|\log(n) +|N|\log\left(1+\frac{|Y|}{|N|}\right) \\
    &\leq |N|\log(n) +|Y|,
\end{align*}
where the last equality again follows from the logarithm inequality. We derive the final inequality by
\begin{align*}
f_{\mathsf{CT}}^D(X,Y,D) &= f_1(X,Y,D)+f_2(X,Y,D)\\
&\leq (1+\log(n))|N|-\sum_{s=1}^r(|Y_{\cdot s}|+1)\log\left(\frac{|Y_{\cdot s}|+1}{|Y|+r}\right)+|Y|
\end{align*}
\qed
\end{proof}
\section[tbl]{Calculating the Lipschitz Moduli of PRIMP} \label{sec:applipschitzmoduli}
We study the partial gradients of the regularization term used in \textsc{Primp} (Sec.~\ref{sec:primp})
\begin{align*}
\nabla_X G(X,Y)&=c(0.5)_s^T\\
\nabla_Y G(X,Y)&=-\frac{1}{2}\left(\log\left(\frac{|Y_{\cdot s}|+1}{|Y|+r}\right)\right)_{js}+(0.5)_{js}.
\end{align*}
The partial gradient with respect to $X$ is constant and has a Lipschitz constant of zero. The partial gradient with respect to $Y$ can be written as the sum
\begin{align*}
\nabla_YG(X,Y)=-\frac{1}{2}(\underbrace{(\log(|Y_{\cdot s}|+1))_{js}}_{=A(Y)}-\underbrace{(\log(|Y|+r))_{js}}_{=B(Y)})+(0.5)_{js}.
\end{align*}
From the triangle inequality follows that the gradient with respect to $Y$ is Lipschitz continuous with modulus $M_{\nabla_YG}(X)=\frac{1}{2}(M_A+M_B)$, if the functions $A$ and $B$ are Lipschitz continuous with moduli $M_A$ and $M_B$:
\begin{align*}
\|\nabla_YG(X,Y)-\nabla_VG(X,V)\|&=\frac{1}{2}\|A(Y)-A(V)+B(Y)-B(V)\|\\
&\leq \frac{1}{2}\|A(Y)-A(V)\|+\|B(Y)-B(V)\|\\
&\leq \frac{M_A+M_B}{2}\|Y-V\|.
\end{align*}
The one-dimensional function $x\mapsto\log(x+\delta)$, $x\in \R_+$ is for any $\delta>0$ Lipschitz continuous with modulus $\delta^{-1}$. This can be easily derived by the mean value theorem and the bound 
\[\frac{d}{dx}\log(x+\delta)=\frac{1}{x+\delta}\leq \frac{1}{\delta}\]
for all $x\geq 0$. We show with the following equations, that $M_A=M_B=m$. For improved readability, we use the squared Lipschitz inequality, i.e.,
\begin{align}
\|A(Y)-A(V)\|^2 &=\sum_{s,j}(\log(|Y_{\cdot s}|+1)-\log(|V_{\cdot s}|+1))^2\nonumber\\
&=m\sum_{s=1}^r(\log(|Y_{\cdot s}|+1)-\log(|V_{\cdot s}|+1))^2\nonumber \\
&\leq m\sum_{s=1}^r(|Y_{\cdot s}|-|V_{\cdot s}|)^2\label{eq:lipLog1}\\
&= m\sum_{s=1}^r\left(\sum_{j=1}^m(Y_{j s}-V_{j s})\right)^2\nonumber\\
&\leq m^2\sum_{s,j}(Y_{j s}-V_{j s})^2= m^2\|Y-V\|^2,\label{eq:cauchySchw1}
\end{align}
where Eq.~(\ref{eq:lipLog1}) follows from the Lipschitz continuity of the logarithmic function as discussed above for $\delta=1$ and Eq.~(\ref{eq:cauchySchw1}) follows from the Cauchy-Schwarz inequality. Similar steps yield the Lipschitz modulus of $B$,
\begin{align}
\|B(Y)-B(V)\|^2 &=\sum_{s,j}(\log(|Y|+r)-\log(|V|+r))^2\nonumber\\
&=mr(\log(|Y|+r)-\log(|V|+r))^2\nonumber\\
&\leq \frac{mr}{r^2}(|Y|-|V|)^2\nonumber\\
&= \frac{m}{r}\left(\sum_{s,j}(Y_{j s}-V_{j s})\right)^2\nonumber\\
&\leq m^2\sum_{s,j}(Y_{j s}-V_{j s})^2.\nonumber
\end{align}
We conclude that the Lipschitz moduli of the gradients are given as
\[M_{\nabla_X G}(Y)=0 \quad M_{\nabla_YG}(X)=m.\]
\end{document}

%% file: header.tex
  \usepackage[table]{xcolor}
  \usepackage{subfigure}
  \usepackage{algorithm,algpseudocode}
  \usepackage{pgfplots,pgfplotstable}
  \usepackage{amsmath,amssymb,upgreek}
  \usepackage{thmtools, thm-restate}
  \usepackage{environ}
  \usepackage{adjustbox}
  \usepackage{tabularx}
  \usepackage{enumerate}
  \pgfplotsset{compat=1.10}
  \usepackage{tikz}
  \usetikzlibrary{matrix,decorations.pathreplacing, calc, positioning}
  \usepackage[english]{babel}
  \usepackage[utf8x]{inputenc}
  \usepackage{nicefrac}
  \usepackage{filecontents}
  \usepackage{booktabs,multirow}
  \usepackage{cases} 
  \usepackage{hhline}
  \usepackage{bbm} 
  \usepackage{url}
  \usepackage{relsize}

\DeclareMathOperator*{\argmin}{arg\,min}

\DeclareMathOperator*{\rank}{rank}

\DeclareMathOperator{\prox}{prox}
\DeclareMathOperator{\dom}{dom}
\DeclareMathOperator{\pre}{pre}
\DeclareMathOperator{\rec}{rec}
\DeclareMathOperator{\I}{\mathcal{I}}

\newcommand{\T}{\mathcal{T}}

\newcommand{\N}{\mathbb{N}}

\newcommand{\R}{\mathbb{R}}

\definecolor{TUGreen}{RGB}{132,184,24}
\definecolor{TUGray}{RGB}{104,104,104}
\definecolor{tblEmph}{RGB}{241,226,204}
\definecolor{mygreen}{RGB}{181,221,109}
\definecolor{myyellow}{RGB}{255,242,174}
\definecolor{myseablue}{RGB}{147,213,198}
\definecolor{mylila}{RGB}{189,132,190}
\definecolor{myorange}{RGB}{246,179,98}
\definecolor{myred}{RGB}{247,129,116} 
\definecolor{cPanpal}{RGB}{56,108,176}
\definecolor{cPrimp}{RGB}{240,2,127}
\definecolor{cPanda}{RGB}{127,201,127}
\definecolor{cMdl4bmf}{RGB}{190,174,212}
\definecolor{cNassau}{RGB}{253,192,134}
\definecolor{myblue}{RGB}{130,176,208}
\definecolor{myminthe}{RGB}{208,235,189}
\definecolor{mypink}{RGB}{251,205,227}

\usepackage{etoolbox}

\makeatletter
\pretocmd\start@align
{%
  \let\everycr\CT@everycr
  \CT@start
}{}{}
\apptocmd{\endalign}{\CT@end}{}{}
\makeatother

\makeatletter
\newcommand\footnoteref[1]{\protected@xdef\@thefnmark{\ref{#1}}\@footnotemark}
\makeatother

\makeatletter
\def\fixedlabel#1#2{%
  \@bsphack%
  \protected@write\@auxout{}%
         {\string\newlabel{#1}{{#2}{\thepage}}}%
  \@esphack}
\makeatother
\makeatletter
\newcommand{\problemtitle}[1]{\gdef\@problemtitle{#1}}
\newcommand{\probleminput}[1]{\gdef\@probleminput{#1}}
\newcommand{\problemquestion}[1]{\gdef\@problemquestion{#1}}
\NewEnviron{learningProblem}{%
  \problemtitle{}\probleminput{}\problemquestion{}
  \BODY
  \par\addvspace{.5\baselineskip}%
  \noindent%
  \begin{tabularx}{\textwidth}{ p{2pc} p{24pc} c}
    \multicolumn{2}{l}{\@problemtitle} \\
    \textbf{Given} & \@probleminput \\
    \textbf{Find} & \@problemquestion
  \end{tabularx}
  \par\addvspace{.5\baselineskip}
}
\makeatother

\usepackage{newfloat}
\usepackage{caption}
\DeclareFloatingEnvironment[fileext=cmh,placement={!ht},name=List]{myfloat}



%% file: pics/tilingFact.tex
\[
\begin{tikzpicture}[baseline=-0.5ex]
   \matrix [matrix of math nodes,left delimiter=(,right delimiter=)] (n) {
1&1&1&1&1 \\
1&0&1&0&1 \\
0&1&1&1&1 \\
1&0&1&0&1 \\
};
\draw[color=myblue,fill=myblue,opacity=0.2] (n-1-2.north west) -- (n-1-5.north east) -- (n-1-5.south east) -- (n-1-2.south west) -- (n-1-2.north west);
\draw[color=myblue,fill=myblue,opacity=0.2] (n-3-2.north west) -- (n-3-5.north east) -- (n-3-5.south east) -- (n-3-2.south west) -- (n-3-2.north west);
\draw[color=mypink,fill=mypink,opacity=0.2] (n-1-3.north west) -- (n-1-3.north east) -- (n-2-3.south east) -- (n-2-3.south west) --(n-1-3.north west);
\draw[color=mypink,fill=mypink,opacity=0.2] (n-1-5.north west) -- (n-1-5.north east) -- (n-2-5.south east) -- (n-2-5.south west) --(n-1-5.north west);
\draw[color=mypink,fill=mypink,opacity=0.2] (n-1-1.north west) -- (n-1-1.north east) -- (n-2-1.south east) -- (n-2-1.south west) --(n-1-1.north west);
\draw[color=mypink,fill=mypink,opacity=0.2] (n-4-3.north west) -- (n-4-3.north east) -- (n-4-3.south east) -- (n-4-3.south west) --(n-4-3.north west);
\draw[color=mypink,fill=mypink,opacity=0.2] (n-4-5.north west) -- (n-4-5.north east) -- (n-4-5.south east) -- (n-4-5.south west) --(n-4-5.north west);
\draw[color=mypink,fill=mypink,opacity=0.2] (n-4-1.north west) -- (n-4-1.north east) -- (n-4-1.south east) -- (n-4-1.south west) --(n-4-1.north west);
\end{tikzpicture}
=
\mathlarger{\theta}\left(
\begin{tikzpicture}[baseline=-0.5ex]
    \matrix [matrix of math nodes,left delimiter=(,right delimiter=)] (n) {
1&1 \\
1&0 \\
0&1 \\
1&0 \\
};
\draw[color=myblue,fill=myblue,opacity=0.2] (n-1-2.north west) -- (n-1-2.north east) -- (n-1-2.south east) -- (n-1-2.south west) -- (n-1-2.north west);
\draw[color=myblue,fill=myblue,opacity=0.2] (n-3-2.north west) -- (n-3-2.north east) -- (n-3-2.south east) -- (n-3-2.south west) -- (n-3-2.north west);
\draw[color=mypink,fill=mypink,opacity=0.2] (n-1-1.north west) -- (n-1-1.north east) -- (n-2-1.south east) -- (n-2-1.south west) --(n-1-1.north west);
\draw[color=mypink,fill=mypink,opacity=0.2] (n-4-1.north west) -- (n-4-1.north east) -- (n-4-1.south east) -- (n-4-1.south west) --(n-4-1.north west);
\end{tikzpicture}
\cdot
\begin{tikzpicture}[baseline=-0.5ex]
    \matrix [matrix of math nodes,left delimiter=(,right delimiter=)] (n) {
1&0&1&0&1 \\
0&1&1&1&1 \\
};
\draw[color=myblue,fill=myblue,opacity=0.2] (n-2-2.north west) -- (n-2-5.north east) -- (n-2-5.south east) -- (n-2-2.south west) -- (n-2-2.north west);
\draw[color=mypink,fill=mypink,opacity=0.2] (n-1-1.north west) -- (n-1-1.north east) -- (n-1-1.south east) -- (n-1-1.south west) --(n-1-1.north west);
\draw[color=mypink,fill=mypink,opacity=0.2] (n-1-3.north west) -- (n-1-3.north east) -- (n-1-3.south east) -- (n-1-3.south west) --(n-1-3.north west);
\draw[color=mypink,fill=mypink,opacity=0.2] (n-1-5.north west) -- (n-1-5.north east) -- (n-1-5.south east) -- (n-1-5.south west) --(n-1-5.north west);
\end{tikzpicture}
\right)
\]

%% file: pics/OverlapFact.tex
\begin{align*}
D=
\begin{tikzpicture}[baseline=-0.5ex]
   \matrix [matrix of math nodes,left delimiter=(,right delimiter=),ampersand replacement=\&] (n) {
1\&1\&1\&0 \\
1\&1\&1\&1 \\
0\&1\&1\&1 \\
};
\draw[color=myblue,fill=myblue,opacity=0.2] (n-1-1.north west) -- (n-1-3.north east) -- (n-2-3.south east) -- (n-2-1.south west) -- (n-1-1.north west);
\draw[color=mypink,fill=mypink,opacity=0.2] (n-2-2.north west) -- (n-2-4.north east) -- (n-3-4.south east) -- (n-3-2.south west) --(n-2-2.north west);
\end{tikzpicture}
&\approx
\begin{tikzpicture}[baseline=-0.5ex]
   \matrix [matrix of math nodes,left delimiter=(,right delimiter=),ampersand replacement=\&] (n) {
1\&.9\&.9\&.1 \\
.7\&1.2\&1.2\&.7 \\
.1\&.9\&.9\&1 \\
};
\draw[color=myblue,fill=myblue,opacity=0.2] (n-1-1.north west) -- (n-1-3.north east) -- (n-2-3.south east) -- (n-2-1.south west) -- (n-1-1.north west);
\draw[color=mypink,fill=mypink,opacity=0.2] (n-2-2.north west) -- (n-2-4.north east) -- (n-3-4.south east) -- (n-3-2.south west) --(n-2-2.north west);
\end{tikzpicture}
\\
&\approx
\begin{tikzpicture}[baseline=-0.5ex]
    \matrix [matrix of math nodes,left delimiter=(,right delimiter=),ampersand replacement=\&] (n) {
1\&0 \\
.6\&.6 \\
0\&1 \\
};
\draw[color=myblue,fill=myblue,opacity=0.2] (n-1-1.north west) -- (n-1-1.north east) -- (n-2-1.south east) -- (n-2-1.south west) -- (n-1-1.north west);
\draw[color=mypink,fill=mypink,opacity=0.2] (n-2-2.north west) -- (n-2-2.north east) -- (n-3-2.south east) -- (n-3-2.south west) --(n-2-2.north west);
\end{tikzpicture}
\cdot
\begin{tikzpicture}[baseline=-0.5ex]
    \matrix [matrix of math nodes,left delimiter=(,right delimiter=),ampersand replacement=\&] (n) {
1\&.9\&.9\&.1 \\
.1\&.9\&.9\&1 \\
};
\draw[color=myblue,fill=myblue,opacity=0.2] (n-1-1.north west) -- (n-1-3.north east) -- (n-1-3.south east) -- (n-1-1.south west) --(n-1-1.north west);
\draw[color=mypink,fill=mypink,opacity=0.2] (n-2-2.north west) -- (n-2-4.north east) -- (n-2-4.south east) -- (n-2-2.south west) -- (n-2-2.north west);
\end{tikzpicture}
\\
D_A=
\begin{tikzpicture}[baseline=-0.5ex]
   \matrix [matrix of math nodes,left delimiter=(,right delimiter=),ampersand replacement=\&] (n) {
1\&1\&1\&0 \\
1\&2\&2\&1 \\
0\&1\&1\&1 \\
};
\draw[color=myblue,fill=myblue,opacity=0.2] (n-1-1.north west) -- (n-1-3.north east) -- (n-2-3.south east) -- (n-2-1.south west) -- (n-1-1.north west);
\draw[color=mypink,fill=mypink,opacity=0.2] (n-2-2.north west) -- (n-2-4.north east) -- (n-3-4.south east) -- (n-3-2.south west) --(n-2-2.north west);
\end{tikzpicture}
&=
\mathlarger{\theta}
\begin{tikzpicture}[baseline=-0.5ex]
    \matrix [matrix of math nodes,left delimiter=(,right delimiter=),ampersand replacement=\&] (n) {
1\&0 \\
.6\&.6 \\
0\&1 \\
};
\draw[color=myblue,fill=myblue,opacity=0.2] (n-1-1.north west) -- (n-1-1.north east) -- (n-2-1.south east) -- (n-2-1.south west) -- (n-1-1.north west);
\draw[color=mypink,fill=mypink,opacity=0.2] (n-2-2.north west) -- (n-2-2.north east) -- (n-3-2.south east) -- (n-3-2.south west) --(n-2-2.north west);
\end{tikzpicture}
\cdot \mathlarger{\theta}
\begin{tikzpicture}[baseline=-0.5ex]
    \matrix [matrix of math nodes,left delimiter=(,right delimiter=),ampersand replacement=\&] (n) {
1\&.9\&.9\&.1 \\
.1\&.9\&.9\&1 \\
};
\draw[color=myblue,fill=myblue,opacity=0.2] (n-1-1.north west) -- (n-1-3.north east) -- (n-1-3.south east) -- (n-1-1.south west) --(n-1-1.north west);
\draw[color=mypink,fill=mypink,opacity=0.2] (n-2-2.north west) -- (n-2-4.north east) -- (n-2-4.south east) -- (n-2-2.south west) -- (n-2-2.north west);
\end{tikzpicture}
\\
D_B=
\begin{tikzpicture}[baseline=-0.5ex]
   \matrix [matrix of math nodes,left delimiter=(,right delimiter=),ampersand replacement=\&] (n) {
1\&1\&1\&0 \\
1\&1\&1\&1 \\
0\&1\&1\&1 \\
};
\draw[color=myblue,fill=myblue,opacity=0.2] (n-1-1.north west) -- (n-1-3.north east) -- (n-2-3.south east) -- (n-2-1.south west) -- (n-1-1.north west);
\draw[color=mypink,fill=mypink,opacity=0.2] (n-2-2.north west) -- (n-2-4.north east) -- (n-3-4.south east) -- (n-3-2.south west) --(n-2-2.north west);
\end{tikzpicture}
&=
\mathlarger{\theta}\left(\mathlarger{\theta}
\begin{tikzpicture}[baseline=-0.5ex]
    \matrix [matrix of math nodes,left delimiter=(,right delimiter=),ampersand replacement=\&] (n) {
1\&0 \\
.6\&.6 \\
0\&1 \\
};
\draw[color=myblue,fill=myblue,opacity=0.2] (n-1-1.north west) -- (n-1-1.north east) -- (n-2-1.south east) -- (n-2-1.south west) -- (n-1-1.north west);
\draw[color=mypink,fill=mypink,opacity=0.2] (n-2-2.north west) -- (n-2-2.north east) -- (n-3-2.south east) -- (n-3-2.south west) --(n-2-2.north west);
\end{tikzpicture}
\cdot \mathlarger{\theta}
\begin{tikzpicture}[baseline=-0.5ex]
    \matrix [matrix of math nodes,left delimiter=(,right delimiter=),ampersand replacement=\&] (n) {
1\&.9\&.9\&.1 \\
.1\&.9\&.9\&1 \\
};
\draw[color=myblue,fill=myblue,opacity=0.2] (n-1-1.north west) -- (n-1-3.north east) -- (n-1-3.south east) -- (n-1-1.south west) --(n-1-1.north west);
\draw[color=mypink,fill=mypink,opacity=0.2] (n-2-2.north west) -- (n-2-4.north east) -- (n-2-4.south east) -- (n-2-2.south west) -- (n-2-2.north west);
\end{tikzpicture}
\right)
\end{align*}

%% file: pics/Lambda.tex
\begin{tikzpicture}
	\begin{axis}[width=200pt,height = 90pt, axis x line=left,ymax=1.1, domain=-0.5:1.5, axis y line=center, tick align=outside,legend pos=outer north east,legend entries={$\Lambda(x)$}]
    	\addplot+[domain=0:1,mark=none,smooth,mygreen,ultra thick,samples=100] (\x,{abs(abs(2*\x-1)-1)});
        \draw[dashed,mygreen,ultra thick] ({axis cs:1,0}|-{rel axis cs:0,0}) -- ({axis cs:1,0}|-{rel axis cs:0,1});
        \draw[dashed,mygreen,ultra thick] ({axis cs:0,0}|-{rel axis cs:0,0}) -- ({axis cs:0,0}|-{rel axis cs:0,1});
        \draw [draw=mygreen, fill=mygreen, ultra thick] 
            (axis cs: 0, 0) circle (2.0pt);
    	\draw [draw=mygreen, fill=mygreen, ultra thick] 
            (axis cs: 1, 0) circle (2.0pt);	
	\end{axis}
\end{tikzpicture}

%% file: pics/synthNoise8_10.tex
\begin{tikzpicture}[baseline]  
		\begin{axis}[
        	axis lines = left,
			title={$D\in\{0,1\}^{800\times1000}$},
            height=.33\linewidth,
            width=.5\linewidth, 
			ylabel={$F$},
			xmax=25.5, ymin=0, ymax=1.1
			]
            \addplot+[cPrimp,mark options ={cPrimp},mark repeat={3}, thick, error bars/.cd,y dir = both, y explicit]  table[x=x,y=y,y error=std] {superPimpLam4_n800m1000F.dat}; 
            \addplot+[cPanpal,mark options ={cPanpal},mark repeat={3}, thick, error bars/.cd,y dir = both, y explicit]  table[x=x,y=y,y error=std] {nmfL1Lam2_n800m1000F.dat};
            \addplot+[cPanda,dashed,mark options ={cPanda},mark repeat={3}, thick, error bars/.cd,y dir = both, y explicit]  table[x=x,y=y,y error=std] {Panda_n800m1000F.dat}; 
            \addplot+[cNassau,dashed,mark options ={cNassau},mark repeat={3}, thick, error bars/.cd,y dir = both, y explicit]  table[x=x,y=y,y error=std] {Nassau_n800m1000F.dat}; 
           \addplot+[cMdl4bmf,dashed,mark options ={cMdl4bmf},mark repeat={3}, thick, error bars/.cd,y dir = both, y explicit]  table[x=x,y=y,y error=std] {Mdl4bmf_n800m1000F.dat};
		\end{axis} 
\end{tikzpicture}
\begin{tikzpicture}[baseline]  
		\begin{axis}[
        	title={$D^T\in\{0,1\}^{1000\times800}$},
        	axis lines = left,
            height=.33\linewidth,
            width=.5\linewidth,
			xmax=25.5, ymin=0, ymax=1.1
			]
            \addplot+[cPrimp,mark options ={cPrimp},mark repeat={3}, thick, error bars/.cd,y dir = both, y explicit]  table[x=x,y=y,y error=std] {superPimpLam4_n1000m800F.dat}; 
            \addplot+[cPanpal,mark options ={cPanpal},mark repeat={3}, thick, error bars/.cd,y dir = both, y explicit]  table[x=x,y=y,y error=std] {nmfL1Lam2_n1000m800F.dat};
            \addplot+[cPanda,dashed,mark options ={cPanda},mark repeat={3}, thick, error bars/.cd,y dir = both, y explicit]  table[x=x,y=y,y error=std] {Panda_n1000m800F.dat};
            \addplot+[cNassau,dashed,mark options ={cNassau},mark repeat={3}, thick, error bars/.cd,y dir = both, y explicit]  table[x=x,y=y,y error=std] {Nassau_n1000m800F.dat}; 
           \addplot+[cMdl4bmf,dashed,mark options ={cMdl4bmf},mark repeat={3}, thick, error bars/.cd,y dir = both, y explicit]  table[x=x,y=y,y error=std] {Mdl4bmf_n1000m800F.dat};
		\end{axis} 
\end{tikzpicture}
\begin{tikzpicture}[baseline]  
		\begin{axis}[
        	axis lines = left,
            height=.33\linewidth,
            width=.5\linewidth,
			xlabel={$p_+=p_-\ [\%]$},     
			ylabel={$r(X,Y)$},
			xmax=25.5, xmin=0, ymin=0, ymax=40.2
			]
            \addplot+[cPrimp,mark options ={cPrimp},mark repeat={3}, thick, error bars/.cd,y dir = both, y explicit]  table[x=x,y=y,y error=std] {superPimpLam4_n800m1000R.dat}; 
            \addplot+[cPanpal,mark options ={cPanpal},mark repeat={3}, thick, error bars/.cd,y dir = both, y explicit]  table[x=x,y=y,y error=std] {nmfL1Lam2_n800m1000R.dat};
            \addplot+[cPanda,dashed,mark options ={cPanda},mark repeat={3}, thick, error bars/.cd,y dir = both, y explicit]  table[x=x,y=y,y error=std] {Panda_n800m1000R.dat};
            \addplot+[cNassau,dashed,mark options ={cNassau},mark repeat={3}, thick, error bars/.cd,y dir = both, y explicit]  table[x=x,y=y,y error=std] {Nassau_n800m1000R.dat}; 
           \addplot+[cMdl4bmf,dashed,mark options ={cMdl4bmf},mark repeat={3}, thick, error bars/.cd,y dir = both, y explicit]  table[x=x,y=y,y error=std] {Mdl4bmf_n800m1000R.dat};
		\end{axis} 
\end{tikzpicture}
\begin{tikzpicture}[baseline]  
		\begin{axis}[
        	axis lines = left,
            height=.33\linewidth,
            width=.5\linewidth,
			xlabel={$p_+=p_-\ [\%]$},    
			xmax=25.5, xmin=0, ymin=0, ymax=40.2,
            legend columns =-1,
            legend entries={Primp,Panpal,Panda+,Nassau,Mdl4bmf},
            legend to name=named
			]
            \addplot+[cPrimp,mark options ={cPrimp},mark repeat={3}, thick, error bars/.cd,y dir = both, y explicit]  table[x=x,y=y,y error=std] {superPimpLam4_n1000m800R.dat}; 
            \addplot+[cPanpal,mark options ={cPanpal},mark repeat={3}, thick, error bars/.cd,y dir = both, y explicit]  table[x=x,y=y,y error=std] {nmfL1Lam2_n1000m800R.dat};
            \addplot+[cPanda,dashed,mark options ={cPanda},mark repeat={3}, thick, error bars/.cd,y dir = both, y explicit]  table[x=x,y=y,y error=std] {Panda_n1000m800R.dat};
            \addplot+[cNassau,dashed,mark options ={cNassau},mark repeat={3}, thick, error bars/.cd,y dir = both, y explicit]  table[x=x,y=y,y error=std] {Nassau_n1000m800R.dat}; 
           \addplot+[cMdl4bmf,dashed,mark options ={cMdl4bmf},mark repeat={3}, thick, error bars/.cd,y dir = both, y explicit]  table[x=x,y=y,y error=std] {Mdl4bmf_n1000m800R.dat};
		\end{axis} 
\end{tikzpicture}
\\

\pgfplotslegendfromname{named}

%% file: pics/synthNoise5_16.tex
\begin{tikzpicture}[baseline]  
		\begin{axis}[
        	axis lines = left,
			title={$D\in\{0,1\}^{500\times1600}$},
            height=.33\linewidth,
            width=.5\linewidth, 
			ylabel={$F$},
			xmax=25.5, ymin=0, ymax=1.1
			]
            \addplot+[cPrimp,mark options ={cPrimp},mark repeat={3}, thick, error bars/.cd,y dir = both, y explicit]  table[x=x,y=y,y error=std] {superPimpLam4_n500m1600F.dat}; 
            \addplot+[cPanpal,mark options ={cPanpal},mark repeat={3}, thick, error bars/.cd,y dir = both, y explicit]  table[x=x,y=y,y error=std] {nmfL1Lam2_n500m1600F.dat};
            \addplot+[cPanda,dashed,mark options ={cPanda},mark repeat={3}, thick, error bars/.cd,y dir = both, y explicit]  table[x=x,y=y,y error=std] {Panda_n500m1600F.dat}; 
            \addplot+[cNassau,dashed,mark options ={cNassau},mark repeat={3}, thick, error bars/.cd,y dir = both, y explicit]  table[x=x,y=y,y error=std] {Nassau_n500m1600F.dat}; 
           \addplot+[cMdl4bmf,dashed,mark options ={cMdl4bmf},mark repeat={3}, thick, error bars/.cd,y dir = both, y explicit]  table[x=x,y=y,y error=std] {Mdl4bmf_n500m1600F.dat};
		\end{axis} 
\end{tikzpicture}
\begin{tikzpicture}[baseline]  
		\begin{axis}[
        	title={$D^T\in\{0,1\}^{1600\times500}$},
        	axis lines = left,
            height=.33\linewidth,
            width=.5\linewidth,
			xmax=25.5, ymin=0, ymax=1.1
			]
            \addplot+[cPrimp,mark options ={cPrimp},mark repeat={3}, thick, error bars/.cd,y dir = both, y explicit]  table[x=x,y=y,y error=std] {superPimpLam4_n1600m500F.dat}; 
            \addplot+[cPanpal,mark options ={cPanpal},mark repeat={3}, thick, error bars/.cd,y dir = both, y explicit]  table[x=x,y=y,y error=std] {nmfL1Lam2_n1600m500F.dat};
            \addplot+[cPanda,dashed,mark options ={cPanda},mark repeat={3}, thick, error bars/.cd,y dir = both, y explicit]  table[x=x,y=y,y error=std] {Panda_n1600m500F.dat};
            \addplot+[cNassau,dashed,mark options ={cNassau},mark repeat={3}, thick, error bars/.cd,y dir = both, y explicit]  table[x=x,y=y,y error=std] {Nassau_n1600m500F.dat}; 
           \addplot+[cMdl4bmf,dashed,mark options ={cMdl4bmf},mark repeat={3}, thick, error bars/.cd,y dir = both, y explicit]  table[x=x,y=y,y error=std] {Mdl4bmf_n1600m500F.dat};
		\end{axis} 
\end{tikzpicture}
\begin{tikzpicture}[baseline]  
		\begin{axis}[
        	axis lines = left,
            height=.33\linewidth,
            width=.5\linewidth,
			xlabel={$p_+=p_-\ [\%]$},     
			ylabel={$r(X,Y)$},
			xmax=25.5, xmin=0, ymin=0, ymax=40.2
			]
            \addplot+[cPrimp,mark options ={cPrimp},mark repeat={3}, thick, error bars/.cd,y dir = both, y explicit]  table[x=x,y=y,y error=std] {superPimpLam4_n500m1600R.dat}; 
            \addplot+[cPanpal,mark options ={cPanpal},mark repeat={3}, thick, error bars/.cd,y dir = both, y explicit]  table[x=x,y=y,y error=std] {nmfL1Lam2_n500m1600R.dat};
            \addplot+[cPanda,dashed,mark options ={cPanda},mark repeat={3}, thick, error bars/.cd,y dir = both, y explicit]  table[x=x,y=y,y error=std] {Panda_n500m1600R.dat};
            \addplot+[cNassau,dashed,mark options ={cNassau},mark repeat={3}, thick, error bars/.cd,y dir = both, y explicit]  table[x=x,y=y,y error=std] {Nassau_n500m1600R.dat}; 
           \addplot+[cMdl4bmf,dashed,mark options ={cMdl4bmf},mark repeat={3}, thick, error bars/.cd,y dir = both, y explicit]  table[x=x,y=y,y error=std] {Mdl4bmf_n500m1600R.dat};
		\end{axis} 
\end{tikzpicture}
\begin{tikzpicture}[baseline]  
		\begin{axis}[
        	axis lines = left,
            height=.33\linewidth,
            width=.5\linewidth,
			xlabel={$p_+=p_-\ [\%]$},    
			xmax=25.5, xmin=0, ymin=0, ymax=40.2,
            legend columns =-1,
            legend entries={Primp,Panpal,Panda+,Nassau,Mdl4bmf},
            legend to name=named
			]
            \addplot+[cPrimp,mark options ={cPrimp},mark repeat={3}, thick, error bars/.cd,y dir = both, y explicit]  table[x=x,y=y,y error=std] {superPimpLam4_n1600m500R.dat}; 
            \addplot+[cPanpal,mark options ={cPanpal},mark repeat={3}, thick, error bars/.cd,y dir = both, y explicit]  table[x=x,y=y,y error=std] {nmfL1Lam2_n1600m500R.dat};
            \addplot+[cPanda,dashed,mark options ={cPanda},mark repeat={3}, thick, error bars/.cd,y dir = both, y explicit]  table[x=x,y=y,y error=std] {Panda_n1600m500R.dat};
            \addplot+[cNassau,dashed,mark options ={cNassau},mark repeat={3}, thick, error bars/.cd,y dir = both, y explicit]  table[x=x,y=y,y error=std] {Nassau_n1600m500R.dat}; 
           \addplot+[cMdl4bmf,dashed,mark options ={cMdl4bmf},mark repeat={3}, thick, error bars/.cd,y dir = both, y explicit]  table[x=x,y=y,y error=std] {Mdl4bmf_n1600m500R.dat};
		\end{axis} 
\end{tikzpicture}
\\

\pgfplotslegendfromname{named}

%% file: pics/synthNoise.tex
\begin{tikzpicture}[baseline]  
		\begin{axis}[
        	axis lines = left,
			title={Uniform Noise},
            height=.33\linewidth,
            width=.34\linewidth, 
			ylabel={$F$},
			xmax=25.2, ymin=0, ymax=1.1
			]
            \addplot+[cPrimp,mark options ={cPrimp},mark repeat={3}, thick, error bars/.cd,y dir = both, y explicit]  table[x=x,y=y,y error=std] {superPimpLam4UniF.dat}; 
            \addplot+[cPanpal,mark options ={cPanpal},mark repeat={3}, thick, error bars/.cd,y dir = both, y explicit]  table[x=x,y=y,y error=std] {nmfL1Lam2UniF.dat};
            \addplot+[cPanda,dashed,mark options ={cPanda},mark repeat={3}, thick, error bars/.cd,y dir = both, y explicit]  table[x=x,y=y,y error=std] {PandaUniF.dat}; 
            \addplot+[cNassau,dashed,mark options ={cNassau},mark repeat={3}, thick, error bars/.cd,y dir = both, y explicit]  table[x=x,y=y,y error=std] {NassauUniF.dat}; 
           \addplot+[cMdl4bmf,dashed,mark options ={cMdl4bmf},mark repeat={3}, thick, error bars/.cd,y dir = both, y explicit]  table[x=x,y=y,y error=std] {Mdl4bmfUniF.dat};
		\end{axis} 
\end{tikzpicture}
\begin{tikzpicture}[baseline]  
		\begin{axis}[
        	title={Negative Noise},
        	axis lines = left,
            height=.33\linewidth,
            width=.34\linewidth,
			xmax=25.5, ymin=0, ymax=1.1
			]
            \addplot+[cPrimp,mark options ={cPrimp},mark repeat={3}, thick, error bars/.cd,y dir = both, y explicit]  table[x=x,y=y,y error=std] {superPimpLam4NegF.dat}; 
            \addplot+[cPanpal,mark options ={cPanpal},mark repeat={3}, thick, error bars/.cd,y dir = both, y explicit]  table[x=x,y=y,y error=std] {nmfL1Lam2NegF.dat};
            \addplot+[cPanda,dashed,mark options ={cPanda},mark repeat={3}, thick, error bars/.cd,y dir = both, y explicit]  table[x=x,y=y,y error=std] {PandaNegF.dat};
            \addplot+[cNassau,dashed,mark options ={cNassau},mark repeat={3}, thick, error bars/.cd,y dir = both, y explicit]  table[x=x,y=y,y error=std] {NassauNegF.dat}; 
           \addplot+[cMdl4bmf,dashed,mark options ={cMdl4bmf},mark repeat={3}, thick, error bars/.cd,y dir = both, y explicit]  table[x=x,y=y,y error=std] {Mdl4bmfNegF.dat};
		\end{axis} 
\end{tikzpicture}
\begin{tikzpicture}[baseline]  
		\begin{axis}[
        	axis lines = left,
			title={Positive Noise},
            height=.33\linewidth,
            width=.33\linewidth, 
			xmax=25.2, ymin=0, ymax=1.1
			]
            \addplot+[cPrimp,mark options ={cPrimp},mark repeat={3}, thick, error bars/.cd,y dir = both, y explicit]  table[x=x,y=y,y error=std] {superPimpLam4PosF.dat}; 
            \addplot+[cPanpal,mark options ={cPanpal},mark repeat={3}, thick, error bars/.cd,y dir = both, y explicit]  table[x=x,y=y,y error=std] {nmfL1Lam2PosF.dat};
            \addplot+[cPanda,dashed,mark options ={cPanda},mark repeat={3}, thick, error bars/.cd,y dir = both, y explicit]  table[x=x,y=y,y error=std] {PandaPosF.dat};
            \addplot+[cNassau,dashed,mark options ={cNassau},mark repeat={3}, thick, error bars/.cd,y dir = both, y explicit]  table[x=x,y=y,y error=std] {NassauPosF.dat}; 
           \addplot+[cMdl4bmf,dashed,mark options ={cMdl4bmf},mark repeat={3}, thick, error bars/.cd,y dir = both, y explicit]  table[x=x,y=y,y error=std] {Mdl4bmfPosF.dat};
		\end{axis} 
\end{tikzpicture}
\begin{tikzpicture}[baseline]  
		\begin{axis}[
        	axis lines = left,
            height=.33\linewidth,
            width=.35\linewidth,
			xlabel={$p_+=p_-\ [\%]$},     
			ylabel={$r(X,Y)$},
			xmax=25.2, xmin=0, ymin=0, ymax=40.2
			]
            \addplot+[cPrimp,mark options ={cPrimp},mark repeat={3}, thick, error bars/.cd,y dir = both, y explicit]  table[x=x,y=y,y error=std] {superPimpLam4UniR.dat}; 
            \addplot+[cPanpal,mark options ={cPanpal},mark repeat={3}, thick, error bars/.cd,y dir = both, y explicit]  table[x=x,y=y,y error=std] {nmfL1Lam2UniR.dat};
            \addplot+[cPanda,dashed,mark options ={cPanda},mark repeat={3}, thick, error bars/.cd,y dir = both, y explicit]  table[x=x,y=y,y error=std] {PandaUniR.dat};
            \addplot+[cNassau,dashed,mark options ={cNassau},mark repeat={3}, thick, error bars/.cd,y dir = both, y explicit]  table[x=x,y=y,y error=std] {NassauUniR.dat}; 
           \addplot+[cMdl4bmf,dashed,mark options ={cMdl4bmf},mark repeat={3}, thick, error bars/.cd,y dir = both, y explicit]  table[x=x,y=y,y error=std] {Mdl4bmfUniR.dat};
		\end{axis} 
\end{tikzpicture}
\begin{tikzpicture}[baseline]  
		\begin{axis}[
        	axis lines = left,
            height=.33\linewidth,
            width=.35\linewidth,
			xlabel={$p_-\ [\%]\ (p_+=3\%)$},    
			xmax=25.2, xmin=0, ymin=0, ymax=40.2
			]
            \addplot+[cPrimp,mark options ={cPrimp},mark repeat={3}, thick, error bars/.cd,y dir = both, y explicit]  table[x=x,y=y,y error=std] {superPimpLam4NegR.dat}; 
            \addplot+[cPanpal,mark options ={cPanpal},mark repeat={3}, thick, error bars/.cd,y dir = both, y explicit]  table[x=x,y=y,y error=std] {nmfL1Lam2NegR.dat};
            \addplot+[cPanda,dashed,mark options ={cPanda},mark repeat={3}, thick, error bars/.cd,y dir = both, y explicit]  table[x=x,y=y,y error=std] {PandaNegR.dat};
            \addplot+[cNassau,dashed,mark options ={cNassau},mark repeat={3}, thick, error bars/.cd,y dir = both, y explicit]  table[x=x,y=y,y error=std] {NassauNegR.dat}; 
           \addplot+[cMdl4bmf,dashed,mark options ={cMdl4bmf},mark repeat={3}, thick, error bars/.cd,y dir = both, y explicit]  table[x=x,y=y,y error=std] {Mdl4bmfNegR.dat};
		\end{axis} 
\end{tikzpicture}
\begin{tikzpicture}[baseline]  
		\begin{axis}[
        	axis lines = left,
            height=.33\linewidth,
            width=.34\linewidth,
			xlabel={$p_+\ [\%]\ (p_-=3\%)$},    
			xmax=25.2, xmin=0, ymin=0, ymax=40.2,
            legend columns =-1,
            legend entries={Primp,Panpal,Panda+,Nassau,Mdl4bmf},
            legend to name=named
			]
            \addplot+[cPrimp,mark options ={cPrimp},mark repeat={3}, thick, error bars/.cd,y dir = both, y explicit]  table[x=x,y=y,y error=std] {superPimpLam4PosR.dat}; 
            \addplot+[cPanpal,mark options ={cPanpal},mark repeat={3}, thick, error bars/.cd,y dir = both, y explicit]  table[x=x,y=y,y error=std] {nmfL1Lam2PosR.dat};
            \addplot+[cPanda,dashed,mark options ={cPanda},mark repeat={3}, thick, error bars/.cd,y dir = both, y explicit]  table[x=x,y=y,y error=std] {PandaPosR.dat};
            \addplot+[cNassau,dashed,mark options ={cNassau},mark repeat={3}, thick, error bars/.cd,y dir = both, y explicit]  table[x=x,y=y,y error=std] {NassauPosR.dat}; 
           \addplot+[cMdl4bmf,dashed,mark options ={cMdl4bmf},mark repeat={3}, thick, error bars/.cd,y dir = both, y explicit]  table[x=x,y=y,y error=std] {Mdl4bmfPosR.dat};
		\end{axis} 
\end{tikzpicture}
\\

\pgfplotslegendfromname{named}

%% file: pics/synthRank.tex
\begin{tikzpicture}[baseline]  
		\begin{axis}[
        	axis lines = left,
            height=.33\linewidth,
            width=.5\linewidth, 
            xlabel={$r^\star$},   
			ylabel={$F$},
			xmax=45.5, ymin=0, ymax=1.1
			]
            \addplot+[cPrimp,mark options ={cPrimp},mark repeat={3}, thick, error bars/.cd,y dir = both, y explicit]  table[x=x,y=y,y error=std] {superPimpLam4F.dat}; 
            \addplot+[cPanpal,mark options ={cPanpal},mark repeat={3}, thick, error bars/.cd,y dir = both, y explicit]  table[x=x,y=y,y error=std] {nmfL1Lam2F.dat};
            \addplot+[cPanda,dashed,mark options ={cPanda},mark repeat={3}, thick, error bars/.cd,y dir = both, y explicit]  table[x=x,y=y,y error=std] {PandaF.dat}; 
            \addplot+[cNassau,dashed,mark options ={cNassau},mark repeat={3}, thick, error bars/.cd,y dir = both, y explicit]  table[x=x,y=y,y error=std] {NassauF.dat}; 
           \addplot+[cMdl4bmf,dashed,mark options ={cMdl4bmf},mark repeat={3}, thick, error bars/.cd,y dir = both, y explicit]  table[x=x,y=y,y error=std] {Mdl4bmfF.dat};
		\end{axis} 
\end{tikzpicture}
\begin{tikzpicture}[baseline]  
		\begin{axis}[
        	axis lines = left,
            height=.33\linewidth,
            width=.5\linewidth,
			xlabel={$r^\star$},     
			ylabel={$r(X,Y)$},
			xmax=45.5, xmin=0, ymin=0,
            legend columns =-1,
            legend entries={Primp,Panpal,Panda+,Nassau,Mdl4bmf},
             legend to name=named
			]
            \addplot+[cPrimp,mark options ={cPrimp},mark repeat={3}, thick, error bars/.cd,y dir = both, y explicit]  table[x=x,y=y,y error=std] {superPimpLam4R.dat}; 
            \addplot+[cPanpal,mark options ={cPanpal},mark repeat={3}, thick, error bars/.cd,y dir = both, y explicit]  table[x=x,y=y,y error=std] {nmfL1Lam2R.dat};
            \addplot+[cPanda,dashed,mark options ={cPanda},mark repeat={3}, thick, error bars/.cd,y dir = both, y explicit]  table[x=x,y=y,y error=std] {PandaR.dat};
            \addplot+[cNassau,dashed,mark options ={cNassau},mark repeat={3}, thick, error bars/.cd,y dir = both, y explicit]  table[x=x,y=y,y error=std] {NassauR.dat}; 
           \addplot+[cMdl4bmf,dashed,mark options ={cMdl4bmf},mark repeat={3}, thick, error bars/.cd,y dir = both, y explicit]  table[x=x,y=y,y error=std] {Mdl4bmfR.dat};
		\end{axis} 
\end{tikzpicture}
\\

\pgfplotslegendfromname{named}

%% file: pics/synthDensity.tex
\begin{tikzpicture}[baseline]  
		\begin{axis}[
        	axis lines = left,
            height=.33\linewidth,
            width=.5\linewidth, 
            xlabel={$q$},
			ylabel={$F$},
			xmax=0.31, ymin=0, ymax=1.1
			]
            \addplot+[cPrimp,mark options ={cPrimp},mark repeat={3}, thick, error bars/.cd,y dir = both, y explicit]  table[x=x,y=y,y error=std] {superPimpLam4F.dat}; 
            \addplot+[cPanpal,mark options ={cPanpal},mark repeat={3}, thick, error bars/.cd,y dir = both, y explicit]  table[x=x,y=y,y error=std] {nmfL1Lam2F.dat};
            \addplot+[cPanda,dashed,mark options ={cPanda},mark repeat={3}, thick, error bars/.cd,y dir = both, y explicit]  table[x=x,y=y,y error=std] {PandaF.dat}; 
            \addplot+[cNassau,dashed,mark options ={cNassau},mark repeat={3}, thick, error bars/.cd,y dir = both, y explicit]  table[x=x,y=y,y error=std] {NassauF.dat}; 
           \addplot+[cMdl4bmf,dashed,mark options ={cMdl4bmf},mark repeat={3}, thick, error bars/.cd,y dir = both, y explicit]  table[x=x,y=y,y error=std] {Mdl4bmfF.dat};
		\end{axis} 
\end{tikzpicture}
\begin{tikzpicture}[baseline]  
		\begin{axis}[
        	axis lines = left,
            height=.33\linewidth,
            width=.5\linewidth,
			xlabel={$q$},     
			ylabel={$r(X,Y)$},
			xmax=0.31,  ymin=0,
            legend columns =-1,
             legend entries={Primp,Panpal,Panda+,Nassau,Mdl4bmf},
             legend to name=named
			]
            \addplot+[cPrimp,mark options ={cPrimp},mark repeat={3}, thick, error bars/.cd,y dir = both, y explicit]  table[x=x,y=y,y error=std] {superPimpLam4R.dat}; 
            \addplot+[cPanpal,mark options ={cPanpal},mark repeat={3}, thick, error bars/.cd,y dir = both, y explicit]  table[x=x,y=y,y error=std] {nmfL1Lam2R.dat};
            \addplot+[cPanda,dashed,mark options ={cPanda},mark repeat={3}, thick, error bars/.cd,y dir = both, y explicit]  table[x=x,y=y,y error=std] {PandaR.dat};
            \addplot+[cNassau,dashed,mark options ={cNassau},mark repeat={3}, thick, error bars/.cd,y dir = both, y explicit]  table[x=x,y=y,y error=std] {NassauR.dat}; 
           \addplot+[cMdl4bmf,dashed,mark options ={cMdl4bmf},mark repeat={3}, thick, error bars/.cd,y dir = both, y explicit]  table[x=x,y=y,y error=std] {Mdl4bmfR.dat};
		\end{axis} 
\end{tikzpicture}
 \\

\pgfplotslegendfromname{named}

%% file: pics/synthRInc.tex
\begin{tikzpicture}[baseline]  
		\begin{axis}[
        	axis lines = left,
			title={},
            height=.33\linewidth,
            width=.5\linewidth, 
            xlabel={$p_+=p_-\ [\%]$},
			ylabel={$F$},
			xmax=25.2, ymin=0, ymax=1.1,
            legend columns =-1,
      		legend to name = named
			]
            \addlegendimage{no marks,dash pattern=on 4pt off 2pt on 4pt off 2pt,  color=black!60}
			\addlegendimage{no marks,dash pattern=on 8pt off 1pt on 8pt off 1pt, color=black!80}
            \addlegendimage{no marks,thick, color=black}
            \addlegendimage{no marks,ultra thick, color=black}
            \addplot+[cPrimp!60,dash pattern=on 4pt off 2pt on 4pt off 2pt,mark=*,mark options ={cPrimp!60},mark repeat={3}, thick, error bars/.cd,y dir = both, y explicit]  table[x=x,y=y,y error=std] {superPimpInc2F.dat};
            \addlegendentry{$\Delta_r=2$};
            \addplot+[cPrimp!80,dash pattern=on 8pt off 1pt on 8pt off 1pt,mark=*,mark options ={cPrimp!80},mark repeat={3}, thick, error bars/.cd,y dir = both, y explicit]  table[x=x,y=y,y error=std] {superPimpInc5F.dat};
            \addlegendentry{$\Delta_r=5$};
            \addplot+[cPrimp,mark=*,mark options ={cPrimp},mark repeat={3}, thick, error bars/.cd,y dir = both, y explicit]  table[x=x,y=y,y error=std] {superPimpLam4F.dat}; 
            \addlegendentry{$\Delta_r=10$};
            \addplot+[cPrimp!70!black,mark options ={cPrimp!70!black},mark=*,mark repeat={3},ultra thick, error bars/.cd,y dir = both, y explicit]  table[x=x,y=y,y error=std] {superPimpInc20F.dat}; 
            \addlegendentry{$\Delta_r=20$};
            \addplot+[cPanpal!60,dash pattern=on 4pt off 2pt on 4pt off 2pt,mark=square*,mark options ={cPanpal!60},mark repeat={3}, thick, error bars/.cd,y dir = both, y explicit]  table[x=x,y=y,y error=std] {nmfL1Inc2F.dat};
            \addplot+[cPanpal!80,dash pattern=on 8pt off 1pt on 8pt off 1pt,mark=square*,mark options ={cPanpal!80},mark repeat={3}, thick, error bars/.cd,y dir = both, y explicit]  table[x=x,y=y,y error=std] {nmfL1Inc5F.dat};
            \addplot+[cPanpal,solid,mark=square*,mark options ={cPanpal},mark repeat={3}, thick, error bars/.cd,y dir = both, y explicit]  table[x=x,y=y,y error=std] {nmfL1Lam2F.dat};
            \addplot+[cPanpal!70!black,solid,mark=square*,mark options ={cPanpal!70!black},mark repeat={3}, ultra thick, error bars/.cd,y dir = both, y explicit]  table[x=x,y=y,y error=std] {nmfL1Inc20F.dat}; 
		\end{axis} 
\end{tikzpicture}
\begin{tikzpicture}[baseline]  
		\begin{axis}[
        	axis lines = left,
            height=.33\linewidth,
            width=.5\linewidth,
			xlabel={$p_+=p_-\ [\%]$},     
			ylabel={$r(X,Y)$},
			xmax=25.2, xmin=0, ymin=0, ymax=35,
            legend columns =-1,
        	legend style={yshift=-3.5cm},  
			legend cell align=center,
            legend to name = algos
			]
            \addlegendimage{only marks, mark=*, color=cPrimp}
			\addlegendimage{only marks, mark=square*, color=cPanpal}
            \addplot+[cPrimp!60,mark=*,dash pattern=on 4pt off 2pt on 4pt off 2pt,mark options ={cPrimp!60},mark repeat={2}, thick, error bars/.cd,y dir = both, y explicit]  table[x=x,y=y,y error=std] {superPimpInc2R.dat};
            \addlegendentry{\textsc{Primp}};
             \addplot+[cPrimp!80,mark=*,dash pattern=on 8pt off 1pt on 8pt off 1pt,mark options ={cPrimp!80},mark repeat={3}, thick, error bars/.cd,y dir = both, y explicit]  table[x=x,y=y,y error=std] {superPimpInc5R.dat};
            \addplot+[cPrimp,mark=*,mark options ={cPrimp},mark repeat={2}, thick, error bars/.cd,y dir = both, y explicit]  table[x=x,y=y,y error=std] {superPimpLam4R.dat};
            \addplot+[cPrimp!70!black,mark=*, mark options ={cPrimp!70!black},mark repeat={2}, ultra thick, error bars/.cd,y dir = both, y explicit]  table[x=x,y=y,y error=std] {superPimpInc20R.dat}; 
         \addplot+[cPanpal!60,dash pattern=on 4pt off 2pt on 4pt off 2pt,mark=square*,mark options ={cPanpal!60},mark repeat={2}, thick, error bars/.cd,y dir = both, y explicit]  table[x=x,y=y,y error=std] {nmfL1Inc2R.dat};
         \addlegendentry{\textsc{Panpal}};
             \addplot+[cPanpal!80,dash pattern=on 8pt off 1pt on 8pt off 1pt,mark=square*,mark options ={cPanpal!80},mark repeat={2}, thick, error bars/.cd,y dir = both, y explicit]  table[x=x,y=y,y error=std] {nmfL1Inc5R.dat};
            \addplot+[cPanpal,solid,mark=square*, mark options ={cPanpal},mark repeat={2}, thick, error bars/.cd,y dir = both, y explicit]  table[x=x,y=y,y error=std] {nmfL1Lam2R.dat};
            \addplot+[cPanpal!70!black,solid,mark=square*, mark options ={cPanpal!70!black},mark repeat={2}, ultra thick, error bars/.cd,y dir = both, y explicit]  table[x=x,y=y,y error=std] {nmfL1Inc20R.dat}; 
		\end{axis} 
\end{tikzpicture}
\\

\pgfplotslegendfromname{algos}
\pgfplotslegendfromname{named}